%% file: arxiv_main.tex
\documentclass[11pt]{article}

\usepackage[utf8]{inputenc} 
\usepackage[T1]{fontenc}
\usepackage{times}

\usepackage{amsmath,amssymb,amsthm}
\usepackage{mathtools}
\usepackage{bm}
\usepackage{dsfont}
\usepackage{bbm}

\usepackage{graphicx}
\usepackage{subfigure}
\usepackage{xcolor}
\usepackage{booktabs}
\usepackage{nicefrac}
\usepackage{microtype}
\usepackage{parskip}
\usepackage{fullpage}

\usepackage{algorithm2e}
\usepackage{float}
\newfloat{Algorithm}{h}{lop}

\usepackage{sidecap}
\sidecaptionvpos{figure}{t}

\usepackage[numbers]{natbib}

\usepackage[colorlinks=true,linkcolor=blue,urlcolor=blue,citecolor=green]{hyperref}
\usepackage[capitalize]{cleveref}

\theoremstyle{plain}
\newtheorem{theorem}{Theorem}[section]
\newtheorem{lemma}[theorem]{Lemma}
\newtheorem{proposition}[theorem]{Proposition}

\theoremstyle{definition}
\newtheorem{definition}[theorem]{Definition}

\theoremstyle{remark}

\newenvironment{keywords}{\par\noindent\textbf{Keywords:} }{\par}
\newcommand{\acks}[1]{\section*{Acknowledgments}#1}

\input{macros}

\newcommand{\RR}{\mathbb{R}}
\newcommand{\EE}{\mathbb{E}}
\newcommand{\PP}{\mathbb{P}}
\newcommand{\NN}{\mathbb{N}}
\newcommand{\OPT}{\mathrm{OPT}}

\newcommand{\calH}{\mathcal{H}}
\newcommand{\calG}{\mathcal{G}}
\newcommand{\lrangle}[1]{\left< #1 \right>}
\newcommand{\lrparenth}[1]{\left( #1 \right)}
\newcommand{\lrbracket}[1]{\left[ #1 \right]}
\newcommand{\abs}[1]{\left| #1 \right|}
\newcommand{\indicator}[1]{\mathbbm{1}\left\{ #1 \right\}}
\newcommand{\bfzero}{\mathbf{0}}

\title{Agnostic Learning of Arbitrary ReLU Activation under Gaussian Marginals}
\author{%
  Anxin Guo\\
  Northwestern University\\
  \texttt{anxinbguo@gmail.com}
  \and
  Aravindan Vijayaraghavan\\
  Northwestern University\\
  \texttt{aravindv@northwestern.edu}
}
\date{}

\begin{document}
\maketitle

\begin{abstract}
    \input{main_abstract}

\end{abstract}

\begin{keywords}
agnostic learning, general ReLU activation, CSQ-SQ separation
\end{keywords}

\section{Introduction}\label{sec:intro}
\input{main_introduction}

\section{Warm-up and preliminaries}\label{sec:main-warm-up-and-prelim}
\input{main_warm_up_and_prelim}

\section{Algorithmic results}\label{sec:main-algorithmic-results}
\input{main_algorithmic_results}

\section{Lower Bound for CSQ algorithms}\label{sec:main-hardness-results}
\input{main_hardness_results}

\acks{We thank Alex Tang and Santosh Vempala for helpful discussions. \bnote{Add?}
Both authors were supported by the National Science Foundation under Grant Nos. ECCS-2216970 and CCF2154100, and also acknowledge the partial support of Adobe Research.
}

\bibliographystyle{plainnat}
\bibliography{ref}

\newpage
\appendix

\crefalias{section}{appendix}
\input{appendix}

\section{Omitted proofs from reweighted PGD(\cref{subsec:main-reweighted-PGD})}\label{sec:appendix-reweighted-PGD}
\input{appendix_reweighted_PGD}

\section{Omitted proofs from thresholded PCA (\cref{subsec:main-thresholded-PCA})}\label{sec:appendix-filtered-PCA}
\input{appendix_filtered_PCA}

\section{Omitted proofs from CSQ hardness (\cref{sec:main-hardness-results})}\label{sec:appendix-hardness-results}
\input{appendix_hardness_results}

\end{document}

%% file: macros.tex
\newtheorem{assumption}[theorem]{Assumption}
\newtheorem{fact}[theorem]{Fact}

\newcommand{\eat}[1]{}

\newcommand{\R}{\mathbb{R}}

\newcommand{\calA}{\mathcal{A}}
\newcommand{\calD}{\mathcal{D}}

\newcommand{\calN}{\mathcal{N}}

\newcommand{\calX}{\mathcal{X}}

\newcommand{\poly}{\mathrm{poly}}

\newcommand{\norm}[1]{\lVert #1 \rVert}

\newcommand{\Esymb}{\mathbb{E}}

\DeclareMathOperator*{\E}{\Esymb}

\newcommand{\eps}{\varepsilon}
\renewcommand{\epsilon}{\varepsilon}

\newif\ifnotes\notesfalse

\ifnotes
\usepackage{color}
\definecolor{mygrey}{gray}{0.50}
\newcommand{\notename}[2]{{\textcolor{red}{\footnotesize{\bf (#1:} {#2}{\bf ) }}}}

\newcommand{\anote}[1]{{\notename{Aravindan}{#1}}}
\newcommand{\bnote}[1]{{\notename{Bob}{#1}}}
\else

\newcommand{\notename}[2]{{}}

\newcommand{\enote}[1]{}
\newcommand{\vnote}[1]{}
\newcommand{\bnote}[1]{}
\newcommand{\anote}[1]{}

\fi

%% file: main_abstract.tex
We consider the problem of learning an arbitrarily-biased ReLU activation (or neuron) over Gaussian marginals with the squared loss objective. Despite the ReLU neuron being the basic building block of modern neural networks, we still do not understand the basic algorithmic question of whether an arbitrary ReLU neuron is learnable in the non-realizable setting. In particular, all existing polynomial time algorithms only provide approximation guarantees for the better-behaved unbiased setting or restricted bias setting. 

Our main result is a polynomial time statistical query (SQ) algorithm that gives the first constant factor approximation for arbitrary bias. It outputs a ReLU activation that achieves a loss of $O(\mathrm{OPT}) + \varepsilon$ in time $\mathrm{poly}(d,1/\varepsilon)$, where $\mathrm{OPT}$ is the loss obtained by the optimal ReLU activation. Our algorithm presents an interesting departure from existing algorithms, which are all based on gradient descent and thus fall within the class of correlational statistical query (CSQ) algorithms. We complement our algorithmic result by showing that no polynomial time CSQ algorithm can achieve a constant factor approximation. Together, these results shed light on the intrinsic limitation of gradient descent, while identifying arguably the simplest setting (a single neuron) where there is a separation between SQ and CSQ algorithms.

%% file: main_introduction.tex
\bnote{Changes: argued dependence on $W$ on page 3;\\
added citation to Zarifis et al. on page 4;\\
removed assumption $|y| \leq B$, rewrote the sample bound, and added an WLOG bounded $|y|$ argument for it on pages 28 - 30;\\
mentioned finite sample estimate in algorithm 1;\\
added linear regression in algorithm 1;\\
also mentioned linear regression in the proof of main theorem (page 35 - 36)}

\bnote{TODO: acknowledgment? Also, should we indicate funding sources? }

The Rectified Linear Unit (ReLU) is a predominant activation function in machine learning. A ReLU neuron has two parameters---a vector $w \in \R^d$ and a bias $b \in \R$---and acts on an input $x \in \R^d$ via
\begin{equation} 
\sigma(\lrangle{x,w} + b), \text{ where for all } z \in \R, \sigma(z) = \max\{z,0\}. 
\end{equation}
Given a distribution $\calD$ over samples $(x,y) \in \R^d \times \R$, the loss incurred by a ReLU function with parameters $w \in \R^d, b \in \R$ on distribution $\calD$ is given by the least squares error
\begin{equation} L(w,b) \coloneq \frac{1}{2} \E_{(x,y) \sim \calD}\Big[ \big(y- \sigma(\lrangle{x,w} + b) \big)^2\Big], \text{ and } \OPT= \min_{\substack{w \in \R^d, b \in \R:\\ \norm{w}_2\leq W}} L(w,b) \label{eq:intro:formulation}
\end{equation}
is the error of the best-fit ReLU function for the given data distribution $\calD$, and 
$W$ is an upper bound on the norm of the weight vector (which one can think of as being a large polynomial in $d$).  
The interesting setting when $\OPT >0$ is called the agnostic, or non-realizable, setting of the problem~\citep{KearnsSS94-agnostic-learning}. 
The goal in agnostic learning is to design an algorithm that takes i.i.d. samples from $\calD$ and outputs a ReLU neuron whose loss is competitive with the best-fit error $\OPT$, e.g., achieving loss $\OPT + \eps$ or $O(\OPT) + \eps$.

The problem of agnostic learning of a ReLU activation, also called ReLU regression, has been studied extensively over the past decade.
 This problem is computationally intractable without additional assumptions on the marginal distribution of $x$. Most algorithmic works study the setting where the marginal distribution on $x$ is a standard Gaussian, or make similar distributional assumptions~\citep{DiakonikolasGKK20-agnostic-learning-unbiased-relu-constant-convex-surrogate, FreiCG20-agnostic-learning-unbiased-relu-OPT2/3, AwasthiTV-ICLR23-ReLUGD}. 

Despite the ReLU neuron being the basic building block of modern neural networks, we still do not understand the basic algorithmic question of whether an arbitrary ReLU neuron is learnable in the non-realizable setting. In particular, we are not aware of any polynomial-time algorithm that achieves a non-trivial multiplicative approximation for the best-fit ReLU with arbitrary bias.\footnote{The existing state-of-the-art algorithms incur approximation factors that depend polynomially on $1/\OPT$ when the bias is very negative~\citep{AwasthiTV-ICLR23-ReLUGD}. } The main question we address is: 

{\em Can we design a polynomial time algorithm that learns an arbitrary ReLU activation under Gaussian marginals that achieves approximately optimal loss of $O(\OPT)$, where $\OPT$ is defined in \eqref{eq:intro:formulation}? }

Over the past decade there have been several algorithmic results in the unbiased setting, i.e., when the bias $b=0$~\citep{goel2019learning, FreiCG20-agnostic-learning-unbiased-relu-OPT2/3, DiakonikolasGKK20-agnostic-learning-unbiased-relu-constant-convex-surrogate, DiakonikolasKTZ22-agnostic-learning-unbiased-activations-GD, GollakotaGKS23-SIM-agnostic-1st, WangZDD23-agnostic-learning-unbiased-activations-constant-SGD-convex-surrogate, ZarifisWDD24-SIM-agnostic-2nd, WangZDD2024-agnostic-gaussian-SIM}. In particular, the algorithm of \citet{DiakonikolasGKK20-agnostic-learning-unbiased-relu-constant-convex-surrogate} gives the first efficient $O(\OPT)$ guarantee by minimizing a convex surrogate loss. On the other hand, there is evidence that $\OPT+\eps$ may be computationally intractable \citep{DiakonikolasKPZ-COLT21-OptimalityOfPolyRegression, DiakonikolasKR23-cryptographic-hardness-Gaussian-ReLU}. To the best of our knowledge, the only prior algorithmic result on learning a negatively biased ReLU neuron is the recent work of \citet{AwasthiTV-ICLR23-ReLUGD}, which handles the moderate-bias setting when $|b| =O(1)$. 
In fact, we are aware of few algorithmic results for agnostic learning of \textit{any linear model} in the arbitrary bias setting. (One notable exception is the line of work initiated by \citet{DiakonikolasKS18a-biased-halfspace-nasty-noise} that gives an $O(\OPT)$ agnostic learning guarantee for general linear halfspaces.) The arbitrary bias setting seems to represent a common challenge across many problems in computational learning theory. We refer to \cref{sec:appendix-related-works} for other related works, including those related to single-index models. 

Our main result gives an affirmative answer to the above question. 


\begin{theorem} [SQ algorithm that gets $O(\OPT)$, informal] \label{thm:intro:main}
There exists a constant $\alpha$, such that for all $W > 0$ the following holds. Let $\calD$ be the joint distribution of $(x,y)\in \RR^d\times \RR$, where the $x$-marginal is $\calN(0, I_d)$. 
Algorithm \ref{alg:full-alg} 
uses $\poly(d, \frac{1}{\eps}, \frac{1}{\delta}, W)$ time and samples from $\calD$, and with probability at least $1 - \delta$, outputs parameters $\hat w\in \RR^d, \hat b\in \RR$, such that:
    \[ L(\hat w, \hat b) \leq \alpha\cdot\inf_{\substack{w \in \R^d, b \in \R:\\ \norm{w}_2\leq W}} L(w, b) + \eps. \]
\end{theorem}

The above theorem gives the desired agnostic learning guarantee in polynomial time for an arbitrary ReLU activation. In addition, it is a proper learner\footnote{We allow our algorithm to output $\bfzero$, by the simple argument that $\bfzero$ is the $L^2$-limit of $\sigma(\lrangle{x,w}+b)$, as $b\to-\infty$.} i.e., it outputs a ReLU activation function that achieves this loss.  The algorithm fits into the Statistical Query (SQ) framework~\citep{Kearns93-SQ-model}, which gives oracle access to statistics $\EE_{(x,y)\sim \calD}[f(x,y)]$ for some user-specified function $f$, up to an error tolerance of $\tau$ (typically, $\tau=1/\poly(d)$ for polynomial sample complexity). Our algorithm consists of two main steps -- an initial phase that finds a coarse initializer $w_0$ to the true vector $v$, and then an iterative procedure based on a reweighted version of projected gradient descent to get the desired error of $O(\OPT)$.\bnote{new} We further note that the time/sample complexity dependence on $W$ is standard, since the upper bound on the weight's norm controls the ``scale'' of our problem. \anote{ removed: necessary: the upper bound on the weight's norm controls the ``scale'' of our problem. In particular, when we simultaneously scale up the parameter $W$ and the $y$-values of distribution $\calD$ by some factor $\lambda$,  then the sample complexity of estimating population expectations (even as easy as $\EE[xy]$) would also increase by a multiplicative factor of $\lambda$. }

Our algorithm presents an interesting departure from existing approaches for agnostic learning with zero or restricted bias, which uses gradient descent \citep{FreiCG20-agnostic-learning-unbiased-relu-OPT2/3,AwasthiTV-ICLR23-ReLUGD,DiakonikolasKTZ22-agnostic-learning-unbiased-activations-GD} and other algorithms that fit in the framework of Correlation SQ (CSQ), where the algorithm is only allowed to query values of the form $\E_{(x,y) \sim \calD}[y f(x)]$, for some function $f$, up to a tolerance $\tau$. In fact, we can prove the following strong lower bound for all CSQ algorithms. 

\begin{theorem}[CSQ lower bound of $\omega(\OPT)$]\label{thm:intro:CSQ}
    There exists a function $F(\eps)$ that goes to infinity as $\eps \to 0$, such that for any $\eps>0$ and any constant $\alpha \ge 1$, there exists a family of distributions with $\OPT\le \eps/\alpha$, under which any CSQ algorithm that can agnostically learn an arbitrary ReLU neuron with loss at most $\alpha \cdot \OPT + \eps$ (as defined in \cref{eq:intro:formulation}) must use either $2^{d^{\Omega(1)}}$ queries or queries of tolerance $d^{-F(\eps)}$.
\end{theorem}
In other words, no efficient CSQ algorithm can achieve error of $O(\OPT)$, since doing so requires either $\text{exp}(d^{\Omega(1)})$ queries (exponential time) or $d^{-F(\epsilon)} = d^{\omega(1)}$ tolerance (superpolynomial
samples) in the CSQ framework. In particular, since many variants of gradient descent (GD) under the $L^2$ loss are captured by the CSQ model, this points to the sub-optimality of GD for learning even a single neuron, and motivates the design of new hybrid algorithms for learning neural networks. 

Theorems \ref{thm:intro:main} and \ref{thm:intro:CSQ} together identify a new problem on which there is a separation between SQ and CSQ algorithms. \citet{chen2020learning} gave the first such separation by designing an SQ algorithm for PAC learning (the realizable setting) that is fixed parameter tractable\footnote{ i.e. in time $f(k)\poly(d)$, where $f(k)$ can grow very fast with $k$, but $\poly(d)$ is independent of $k$. } for learning a neural network with $k$ neurons under Gaussian marginals, where superpolynomial $d^{\Omega(k)}$ lower bounds for CSQ algorithms were known~\citep{DiakonikolasKKZ20-CSQ-bound-depth2-NN}. Such a separation has also been identified for learning sparse polynomials~\citep{KianiLLJ024-SQCSQ-separation-2, AndoniPV014-SQCSQ-separation-1} and planted multi or single-index models~\citep{DamianLS22-CSQ-hardness-polynomial-of-one-dimension, DamianPLB24-single-index-optimal-SQ-lower-bound}, both under Gaussian marginals. 
Our two-phase algorithm is also inspired by \citet{ChenM-COLT20-FilteredPCA} on learning certain multi-index models using a two-step non-CSQ algorithm. However, these works are all in the realizable setting ($\OPT=0$). The agnostic setting that we study in this work introduces different challenges that require new algorithmic ideas that we describe in more detail in Section~\ref{sec:main-algorithmic-results}. 
Our problem of agnostic learning of a single ReLU neuron is arguably the simplest setting, and the first agnostic setting where such separation has been identified. \anote{Modified:}

{\em Subsequent work:} In very recent independent work, \citet{zarifis2025robustly} gave a different algorithm that applies a smoothing operator for learning neurons, and extended the $O(OPT)$ approximation guarantee to more general monotone Lipschitz activations. 
However, they do not prove the CSQ lower bound and SQ-CSQ separation that we establish in the biased ReLU setting.   


%% file: main_warm_up_and_prelim.tex
In this section, we first introduce some auxiliary results that simplify our problem. Then, we present intuitions for the unique challenges in our regime, and why previous methods fail. 

\subsection{Preliminaries}

Let $L(w, b) = \frac{1}{2}\EE[(y - \sigma(\lrangle{x,w} + b))^2]$ be the loss of a ReLU neuron with weight $w\in\RR^d$ and bias $b\in\RR$. In our proofs, we often compare our loss $L(w,b)$ to $L(v,b)$, where $v$ is the unit vector minimizing this loss for the given $b$. This serves as a proxy for $\OPT$ in the ``normalized'' problem defined below.

We use $\varphi$ and $\Phi$ to denote the standard Gaussian pdf and cdf, respectively, and $\int f(z)\,d\Phi(z)$ denotes integration of $f$ with respect to the one-dimensional standard Gaussian measure. For any vector $w\in \RR^d$, we use $\|w\|$ to denote its $\ell_2$ norm.

\textbf{Normalizing target vector $v$.}
Suppose the best-fit ReLU for $\calD$ is some $\sigma(\lrangle{x,v} + b)$ with $\|v\|\neq 1$. Then, we can instead work with the normalized problem $(x,\hat{y})\sim \hat{\calD}$, where $\hat{y} = y / \|v\|$. Let $\hat{L}$ be the loss function under $\hat{\calD}$. We prove in \cref{appendix:proof-prelims} the following claim: if we \textit{know} the value $\|v\|$, then scaling by $1 / \|v\|$ reduces the problem to unit vector case, with at most polynomial runtime overhead.
\begin{proposition}\label{prop:prelim:scaling-v}
    In the above setting, the optimal ReLU for $\hat{\calD}$ is $\sigma(\lrangle{x, \hat{v}} + \hat{b})$, where $\hat{v} = v / \|v\|$ is a unit vector, $\hat b = b/\|v\|$, and it has loss $\widehat{\OPT} = \hat{L}(\hat v,\hat b) = \OPT / \|v\|^2$. Moreover, suppose parameters $w,b_w$ incur loss $\hat{L}(w, b_w) \leq \alpha\cdot\widehat{\OPT} + \eps$ on the normalized problem $\hat{\calD}$. Then, the pair $(\|v\|w, \|v\|b_w)$ would incur loss $L(\|v\|w, \|v\|b_w) \leq \alpha\cdot \OPT + \|v\|^2\eps$ on the original problem $\calD$.
\end{proposition}

\textbf{``Guessing'' $\|v\|$ and $b$.} 
To apply appropriate scaling, we must have some knowledge about $\|v\|$. To this end, we apply grid search over various guesses $(\beta, \gamma)$ of the values $(\|v\|, b)$, where $v,b$ are the parameters of the optimal ReLU over $\calD$. (See Algorithm \ref{alg:full-alg}.) This approach is effective as long as the approximation factor w.r.t. an ``estimated optimal ReLU'' is at most a constant multiple of that w.r.t. the true $\OPT = L(v,b)$. We formalize this claim in the next proposition, also proven in \cref{appendix:proof-prelims}: 
\begin{proposition}\label{prop:prelims:search-v-b}
    Let $\hat v$ be the unit vector in $v$'s direction. Let $\delta_v = \beta - \|v\|$ and $\delta_b = \gamma - b$ denote the additive errors of our parameter estimations. Then if $|\delta_v|, |\delta_b| \leq 0.1\sqrt{\eps}$, we have:
    \[ \alpha \cdot L(\beta \hat v, \gamma) \leq O(\alpha)\cdot L(v, b) + O(\eps). \]
\end{proposition}

The time/sample complexity of this grid search will be analyzed in \cref{thm:complete-alg-result} in the appendix. For now, we assume the problem is normalized, the optimal ReLU has $\|v\| = 1$, and $b$ is known. This is sufficient to bound the loss $L(w, b)$, as long as $L(v,b)$ is a good proxy for $\OPT$.




\subsection{Challenges with significant negative bias}

In all previous works where $b$ is not too negative, the optimal ReLU $\sigma(\lrangle{x,v}+b)$ is \textit{linear} on a considerable portion of inputs. Intuitively, this shape should be easier to learn than a small wedge at the far end of the number line, which happens when $b$ is very negative. Our results corroborate this intuition: we can apply previous algorithms when $b$ is bounded below, but as $b\to -\infty$, the CSQ hardness (\cref{thm:intro:CSQ}) takes effect and a specialized algorithm for this limit is needed. 

\textbf{Structural observations.} The $b\to -\infty$ regime comes with a benefit: we can now apply asymptotic analysis and make claims that hold \textit{when the optimal ReLU has sufficiently negative $b$}. Most importantly, we have the following two lemmas, which are proven in \cref{appendix:proof-prelims}: 

\begin{lemma}\label{lem:upper-bounding-OPT}
    Suppose $\alpha \geq 3$. Let $\sigma(\lrangle{x,v}+b)$ be the optimal ReLU with loss $\OPT$, where $\|v\|=1$ and $b$ sufficiently negative. If the zero function incurs loss at least $\alpha\cdot\OPT + \eps$, then we have
    \[ \OPT < \frac{3\Phi(b)}{\alpha b^2}\text{, and }\eps < \frac{3\Phi(b)}{b^2}. \]
\end{lemma}


The intuition behind the two lemmas is the following: if the optimal ReLU has very negative $b$, then the zero function should be a reasonable approximation. Thus, any non-trivial setting (where $\bfzero$ is not good enough) would have $\OPT$ and $\eps$ both be extremely small.



\textbf{An Analysis of (projected) gradient descent.} 
Our new algorithm is motivated by the need to overcome the failure of the following gradient descent (GD) algorithm, as $b\to -\infty$:
\[ (w_{t+1}, b_{t+1}) \gets (w_t, b_t) - \eta \nabla_{w,b} L(w_t, b_t), \text{ for }t=1,\ldots, T. \]

Consider the change in \textit{direction} of $w_{t+1}$, compared to that of $w_t$. At iteration $t$, $w_t$ updates in the following direction: 
\[ -\nabla_w L(w_t,b_t) = \EE[(y - \sigma(\lrangle{x,w_t}+b))\cdot \indicator{\lrangle{x,w_t}\geq -b_t} \cdot x], \]
where the indicator is due to the derivative of the ReLU function: $\sigma'(z) = \indicator{z\geq 0}$. 

Let $v^\perp$ be the component of $v$ that is perpendicular to $w_t$.\footnote{When $w_t$ is a unit vector, we have $v^\perp = \frac{v - \lrangle{v,w_t}w_t}{\|v - \lrangle{v,w_t}w_t\|}$.} 
Observe that, for $w_t$'s direction to approach that of $v$, the update $-\nabla_w L(w_t,b_t)$ must have a \textit{positive component} on $v^\perp$:
\begin{align}
    \big \langle-\nabla_w L(w_t,b_t), v^\perp \big\rangle 
    & = \EE\big[ \big( y - \sigma(\lrangle{x,w_t}+b_t) \big) \cdot \indicator{\lrangle{x,w_t}\geq |b_t|} \cdot \langle x, v^\perp \rangle \big] \nonumber\\
    & = \EE[ y \cdot  \indicator{\lrangle{x,w_t}\geq |b_t|}\cdot \langle x, v^\perp \rangle  ] > 0. \label{eq:intro:GD-making-progress}
\end{align}

This can be viewed as a \textit{conditioned} correlation between $y$ and $\langle x, v^\perp \rangle$, where we call $\indicator{\lrangle{x,w_t} \geq |b_t|}$ the \textit{condition function}. We can then write $y = \sigma(\lrangle{x,v}+b) - (\sigma(\lrangle{x,v}+b) - y)$, where the first term comes from ReLU and the second from noise. Through these lens, GD makes progress when ReLU (conditionally) correlates with $\langle x, v^\perp \rangle$ more than noise does. 




\textbf{One-dimensional adversarial example.} We now give a simple one-dimensional example to show how the correlation between $\langle x, v^\perp \rangle$ and $y$ can point to the \textit{opposite direction}, when $b$ is sufficiently negative. Fix an $\alpha$, let $\OPT = \Phi(b) / 100\alpha b^2$ such that $\bfzero$ is not an $\alpha$-approximation. Consider the following distribution which has loss exactly $\OPT$, illustrated in \cref{fig:intro:1D-GD-bad-example}:
\begin{equation}\label{eq:intro:1D-GD-bad-example}
    y = \begin{cases}
    \sigma(x + b) & \text{ when } x \geq 0, \\
    \sqrt{2\OPT} & \text{ when } x < 0.
\end{cases}
\end{equation} 

\begin{figure}[h]
\centering
\subfigure[Moderate bias.]{\includegraphics[width=0.4\textwidth]{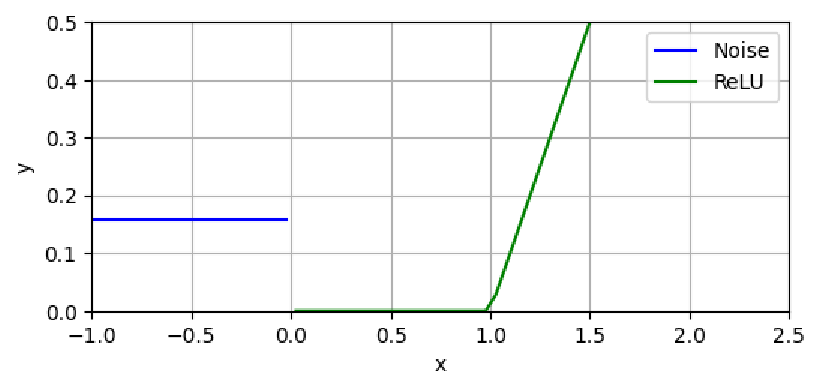}}
\subfigure[Larger bias.]{\includegraphics[width=0.4\textwidth]{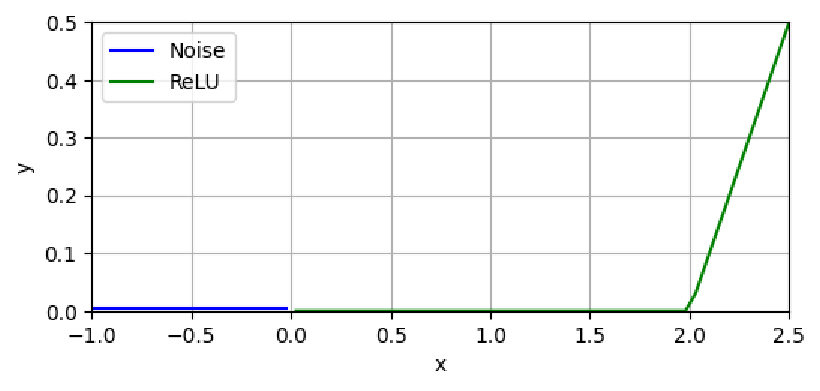}}
    \caption{One-dimensional bad example in \cref{eq:intro:1D-GD-bad-example}. The noise (blue) decreases exponentially in $|b|$. }\label{fig:intro:1D-GD-bad-example}
\end{figure}

Two seemingly contradictory facts are true about this distribution:
\begin{enumerate}
    \item The ReLU produces more \textit{loss} than noise: the zero function incurs about $\frac{\Phi(b)}{b^2} \approx 100\alpha\cdot \OPT$ loss when $x \geq 0$, by not fitting the ReLU. 
    \item The noise dominates the \textit{correlation} with $x$, pulling $\EE[xy]$ towards the negative direction:
    \[ \begin{cases}
        \EE[\sigma(x+b)\cdot x] &= \Phi(b), \\
        \EE[(y - \sigma(x+b))\cdot x] &= -\Omega(\sqrt{\OPT}) = -\Tilde{\Omega}\big(\sqrt{\Phi(b)}\big).
    \end{cases} \]
\end{enumerate}

An intuitive explanation is that the loss scales \textit{quadratically} with $y$, yet the correlation scales (roughly) linearly. This issue appears when $y$ is very close to zero, possible only as $b\to -\infty$.
 
\textbf{Higher dimensions.}
Returning to high-dimensional ReLU regression, we demonstrate how the aforementioned adversarial example can occur in a way that harms GD, per the analysis in \cref{eq:intro:GD-making-progress}. 

Suppose we have some estimated weight $w_t$ and bias $b_t$. Consider the plane spanned by $w_t$ and $v^\perp$ in \cref{fig:intro:gradient-when-bias-small-vs-large}. By removing the uncolored region through the condition function, the direction of correlation in \cref{eq:intro:GD-making-progress} essentially depends on the thin green strip of width $\frac{1}{|b_t|^{0.99}}$, due to Gaussian decay. 

\begin{figure}[h]
\centering
\subfigure[Moderate bias.]{\includegraphics[width=0.35\textwidth]{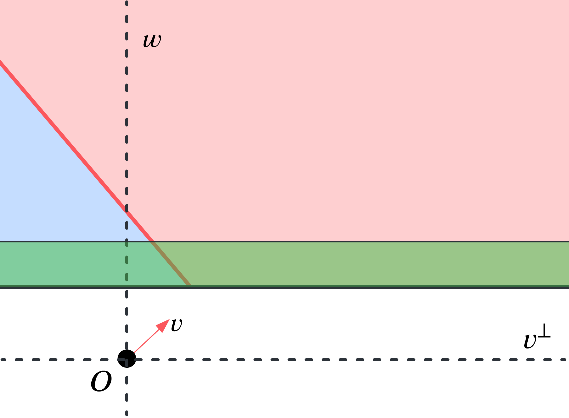}}
\subfigure[Larger bias.]{\includegraphics[width=0.35\textwidth]{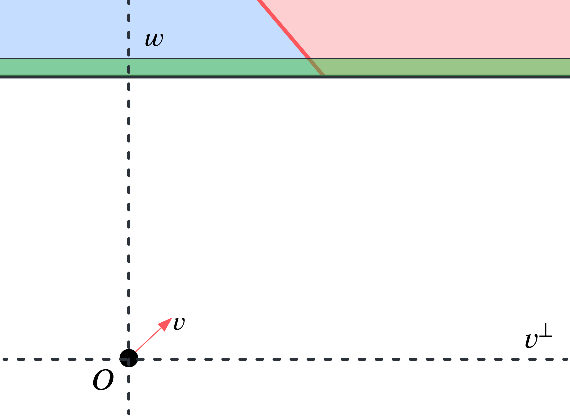}}
    \caption{Plan view of the $w_t$-$v^\perp$ plane in high dimensional analysis of GD. The ReLU is positive on the red region.
    $1-o(1)$ probability mass of the colored region falls in the green strip. }\label{fig:intro:gradient-when-bias-small-vs-large}
\end{figure}

We thus identify a reason for the failure of GD
: the red region is too biased in the green strip, so ReLU may fail to dominate the correlation (\cref{eq:intro:GD-making-progress}), and GD may update in the wrong direction.\footnote{Once exception is when $b_t \leq \frac{b}{\lrangle{w_t, v}}$. But then $b_t$ is too negative and other problems start to occur. } We note that this is strong evidence that \textit{any} GD variant would fail at this task, even those not captured by the CSQ model (\cref{thm:CSQ}). 

%% file: main_algorithmic_results.tex
\begin{Algorithm}[H]
    \centering
    \par
    \noindent\fbox{
    \parbox{0.969\textwidth}{%
        \textbf{Input}: Accuracy parameter $\eps$; sample access to $\calD$; norm upper bound $W$. 
        \begin{enumerate}
            \item Let $h_0 = \bfzero$ be the zero function, and let $L_0 = \EE_\calD[y^2]$ be its loss. 
            \item Run a linear regression algorithm with bias, in the domain $b\geq O\big(\sqrt{\log (1/\eps)}\big)$, and obtain parameters $w_{\mathrm{lin}}\in \RR^d, b_{\mathrm{lin}}$. Let $h_1$ be the ReLU neuron $x\mapsto \sigma(\lrangle{x, w_{\mathrm{lin}}} + b_{\mathrm{lin}})$. 
            \item Run the GD algorithm from \cite{AwasthiTV-ICLR23-ReLUGD} with accuracy $\eps$, and obtain ReLU neuron $h_2$ with loss $L_2$. This is competitive against the best \textit{moderately-biased} ReLU. 
            \item Run grid search over parameter space $[0,W]\times [-\Theta\big(\sqrt{\log(1/\eps)}\big), \Theta\big(\sqrt{\log(1/\eps)}\big)]$ with accuracy $0.1\sqrt{\eps}$, where each pair of parameters $\{(\beta_i, \gamma_i):i\in 3,\ldots,N\}$ indicate ``guesses" for $(\|v\|, b)$. 
            
            \vspace{0.2em}
            For each pair $(\beta_i, \gamma_i)$, we apply the following two subroutines on the normalized distribution $\hat \calD$ (\cref{prop:prelim:scaling-v}), and obtain ReLU $h_i$ with loss 
            \[ L_{i} \leq 0.1\alpha\cdot \OPT_{i} + 0.1 \frac{\eps}{\beta_i^2}, \]
            where $\OPT_{i}$ is the minimum loss among ReLUs $\sigma(\lrangle{x,v}+b)$ with $\|v\| = \beta_i$ and $b = \gamma_i$. 

            \begin{enumerate}
                \item \textit{Thresholded PCA}: draw $m=\mathrm{poly}(\log d, \log 1/\delta, 1/\eps)$ samples $S = \{(x_i, y_i)\}_{i=1}^m$. Output $w_0$, the top eigenvector of the following matrix:
                \[ \EE_{(x,y)\sim S}\left[ x x^\top \cdot \indicator{|y| \geq \frac{1}{|b|}} \right]. \]

                \item \textit{Reweighted PGD}: an iterative GD-like algorithm with initialization $w_0$ and step size $\eta$ being a fixed polynomial of $\eps$, for $T = \poly(1/\eps)$ steps. On each iteration $t\in[T]$, we draw \bnote{double-check} $m=\mathrm{poly}(d, 1/\delta, 1/\eps)$ samples and truncate their $\|x\|$ and $y$ (see \cref{appendix:PGD-making-progress}) if necessary, resulting in training samples $S=\{(x_i, y_i)\}_{i=1}^m$. 
                
                Then, we update $w_t$ by $w_{t+1} \gets \frac{w_t + \eta v_\mathrm{update}}{\|w_t + \eta v_\mathrm{update}\|}$, where
                \[ v_\mathrm{update} = \big(I_d - w_t w_t^\top\big) \EE_{(x,y)\sim S}\left[f_{t}(x)\cdot y \cdot x \cdot \indicator{\lrangle{x,w_t} \geq \Tilde{b}} \right]. \]
                Here, $f_t(x)=\exp\left(\rho|b|\lrangle{x,w_t} - \frac{\rho^2b^2}{2}\right)$, and $\Tilde{b}=(\rho+\lambda-\lambda^2\rho)|b|$. Both choices will be explained and the parameters $\lambda,\rho$ specified during analysis in \cref{subsec:main-reweighted-PGD}. 

                We will set the final output $h_i$ to be the ReLU neuron $x\mapsto \sigma(\lrangle{x,w_t} + b)$ with the smallest loss, for some $t\in [T]$. 
            \end{enumerate}
            
        \end{enumerate}
        \textbf{Output}: $h_{\mathrm{ind}}$, where $\mathrm{ind} = \arg\min_i\{L_i: i\in [N]\}$. 
    }
    }
    \vspace{0.3em}
    \par
    \caption{Complete algorithm, assuming population access to $\calD$.}\label{alg:full-alg}
\end{Algorithm}

In this section, we present our full algorithm (Algorithm \ref{alg:full-alg}), which in polynomial time outputs a ReLU neuron with at most $\alpha\cdot\OPT + \eps$ loss, as defined in \cref{eq:intro:formulation}. We note that our algorithm fits in the SQ model, and for the sake of exposition, the description of Algorithm \ref{alg:full-alg} assumes access to population expectations over $\calD$. We account for finite-sample errors in the appendix. 

Algorithm \ref{alg:full-alg} is a two-stage process consisting of \textit{thresholded principal component analysis (PCA)} and \textit{reweighted projected gradient descent (PGD)}.\footnote{
This is also known as ``Riemannian GD'' in literature, as the unit sphere is a Riemannian manifold. } We will analyze reweighted PGD before thresholded PCA, as the former is intimately related to our observations above. Our algorithm is reminiscent of the two-stage algorithm used in \citet{ChenM-COLT20-FilteredPCA} for polynomials in few relevant directions, but we note that we work under the more challenging agnostic setting, and we use a complete different iterative update algorithm. 

All our formal lemmas and theorems are under the following assumption: 
\begin{assumption}\label{assumption:standard}
Let $\alpha$ be a sufficiently large constant. Let $\calD$ be the distribution of $(x,y)$ over $\RR^d\times \RR$, where the $x$-marginal is standard Gaussian. Let $\OPT$ be the loss of the best-fit ReLU $\sigma(\lrangle{x,v}+b)$ with $\|v\|=1$, and we assume the zero function $\bfzero$ has loss at least $\alpha\cdot \OPT + \eps$. 
\end{assumption}

As described in \cref{sec:main-warm-up-and-prelim}, the assumption $\|v\|=1$ is essentially without loss of generality. In the next two subsections we present main theorems for the two subroutines. The effectiveness of the complete Algorithm \ref{alg:full-alg} is a natural consequence of the two and is proven in \cref{sec:appendix-putting-things-together}. 



\subsection{Reweighted PGD}\label{subsec:main-reweighted-PGD}

Recall the example in \cref{fig:intro:gradient-when-bias-small-vs-large}, that ``correlation'' between $\lrangle{x, v^\perp}$ and $y$ can be dominated by noise instead of the optimal ReLU, hence GD updates in the wrong direction. To overcome this, when calculating the ``gradient'' on each iteration, we will slightly modify the \textit{condition function} from $\mathbbm{1}{\lrangle{x,w_t}+b\geq 0}$ to be $\mathbbm{1}\{\lrangle{x, w_t} \geq \Tilde{b}\}$ for some carefully chosen $\Tilde{b}$. We also add a \textit{weight function} $f_t(x)$ inside the expectation, defined as the change-of-measure function between two Gaussians. Specifically, it shifts the $x$-marginal to be centered around $\rho \cdot |b|\cdot w_t$, for some carefully chosen $\rho\in (0, 1)$. The exact expression of $f_t(x)$ and $\Tilde{b}$ are specified in Algorithm \ref{alg:full-alg}. The parameters $\lambda$ and $\rho$ are constants to be determined during the analysis. 


The insight behind reweighting\footnote{Reweighting is conceptually similar to \textit{rejection sampling}, a method used for learning general linear halfspaces in \citet{DiakonikolasKS18a-biased-halfspace-nasty-noise, DiakonikolasKLZ24-testable-learning-biased-halfspace}. Compared to their approaches, we have to shift the Gaussian mean by a more carefully chosen amount as well as apply a conditioning indicator, due to the nature of our regression problem. }
is the following: suppose $w_t$ is close to $v$ in direction, then shifting the Gaussian center along $w_t$ would essentially \textit{reduce the problem back to the moderate bias} case. This is clear from \cref{fig:intro:gradient-when-bias-small-vs-large}, where shifting the origin upwards in subfigure (b) would result in a similar landscape as subfigure (a). 
We note that this needs $w_t$ to be mostly aligned with $v$ in the first place, hence the necessity of thresholded PCA. 

Below is our main theorem for reweighted PGD. 

\begin{theorem}\label{thm:PGD}
    Suppose \cref{assumption:standard} holds with $b$ sufficiently negative, and suppose the initialization $w_0$ has $\lrangle{w_0, v} \geq 0.9$. Then, reweighted PGD with $\rho = 0.5, \lambda = 0.9$, and $\eta$ being a fixed polynomial of $\eps$ can with probability $1-\delta$ output a unit vector $w$ such that:
    $$ L(\lrangle{x,w} + b) \leq \alpha\cdot \OPT + \eps, $$
    using $\poly(d, 1/\eps, 1/\delta)$ time and samples per iteration, and in $\log (1/\eps)$ iterations. 
\end{theorem}

We prove that reweighting and conditioning allows the ReLU to overcome the noise and guide the algorithm to the correct direction. Contributions from ReLU and noise are measured by the vectors:
\[ \begin{cases}
    v_{\mathrm{relu}} & = \big(I_d - w_t w_t^\top\big) \EE \left[f_t (x)\cdot x \cdot \sigma(\lrangle{x,v}+b) \cdot \mathbbm{1}\{ x_w \geq \Tilde{b} \} \right],\text{ and } \\ 
    v_{\mathrm{noise}} & = \big(I_d - w_t w_t^\top\big) \EE \left[f_t (x)\cdot x \cdot \big(y - \sigma(\lrangle{x,v}+b)\big) \cdot \mathbbm{1}\{ x_w \geq \Tilde{b} \} \right].
\end{cases} \]

We also prove in \cref{appendix:proof-lem-lower-bound-w-v} the following, that closeness of direction implies small loss:
\begin{lemma}\label{lem:lower-bound-w-v}
    Suppose \cref{assumption:standard} holds. For any unit vector $w$, if $L(w, b) > \alpha\cdot \OPT + \eps$, then for all sufficiently negative $b$:
    \[ \|w-v\|^2 \geq \Omega \left( \frac{\alpha\cdot \OPT + \eps}{\Phi(b)} \right). \]
\end{lemma}
\subsubsection{Contribution from the optimal ReLU}
Let $x_w = \lrangle{x, w}$ and $x^\perp = \langle x, v^\perp \rangle$, which simplifies the notation in the following proof. We will focus on the plane spanned by $w$ and $v^\perp$, since the other directions are all irrelevant to $v_{\mathrm{relu}}$. \bnote{added this and the first equality below.} In fact, by coordinate-wise independence of $\calN(0, I_d)$, the direction of $v_{\mathrm{relu}}$ completely aligns with $v^\perp$.  
\begin{lemma}\label{lem:main-PGD-lower-bound-ReLU}
    Suppose \cref{assumption:standard} holds. Let $\rho,\lambda \in (0,1)$ be constants with $\rho\lambda \geq \frac{1}{2}$. If on iteration $t$ we have $\lrangle{w_t,v}\geq \lambda$ and $L(w_t, b) > \alpha\cdot \OPT$, then for all sufficiently negative $b$, we have:
    \begin{align}
        \|v_{\mathrm{relu}}\| = \lrangle{v_\mathrm{relu}, v^\perp} = \Omega\big(\sqrt{\alpha \cdot \OPT + \eps} \big) \cdot \frac{ e^{- \frac{b^2}{2} \big((1-\lambda \rho)^2 - \frac{1}{2}\big) }}{|b|^{3/2}}. \label{eq:PGD-realizable-contribution}
    \end{align}
\end{lemma}
\begin{proof}
    The first equality follows from the coordinate-wise independence of standard Gaussian. 
    
    Note that $f_t(x)$ is really a function of $x_w$, so we write it as $f_t(x_w)$ in this proof. Using the coordinate system of $w_t$ and $v^\perp$, we can write the desired expression as:
    \begin{align*}
        & \EE \big[ f_t(x) \cdot x^\perp \cdot \sigma(\lrangle{x,v} + b)\indicator{x_w \geq (\rho + \lambda - \lambda^2\rho)|b|} \big] \\ 
        & = \int_{(\rho + \lambda - \lambda^2\rho) |b|}^\infty f_t(x_w) \int_{-\infty}^\infty x^\perp \cdot \sigma(\lrangle{x,v}+b) \, d\Phi(x^\perp)\,d\Phi(x_w) \\
        & \geq \int_{(\rho + \lambda - \lambda^2\rho) |b|}^\infty f_t(x_w) \int_{0}^\infty x^\perp \cdot \sqrt{1 - \lrangle{w_t,v}^2} \cdot \sigma\Big(x^\perp - \sqrt{1-\lambda^2}(1-\lambda\rho)|b| \Big) \, d\Phi(x^\perp)\,d\Phi(x_w) \\
        & = \sqrt{1 - \lrangle{w_t,v}^2} \int_{(\lambda - \lambda^2\rho) |b|}^\infty \Phi\Big( \sqrt{1-\lambda^2}(1-\lambda\rho) b\Big) \,d\Phi(x_w) \\
        & = \sqrt{1 - \lrangle{w_t,v}^2} \cdot \Phi\Big( \sqrt{1-\lambda^2}(1-\lambda\rho) b\Big)\cdot \Phi \big( \lambda(1-\lambda\rho)b \big) \\
        & = \Omega(\|w-v\|) \cdot \frac{\varphi \Big( \sqrt{1-\lambda^2}(1-\lambda\rho) b\Big) \cdot \varphi \big( \lambda(1-\lambda\rho)b \big)}{\sqrt{1-\lambda^2} \lambda(1-\lambda\rho)^2 b^2} .
    \end{align*}
    Here, the inequality is due to \cref{lem:technical-relizable-update} in \cref{sec:appendix-reweighted-PGD}, and the final estimation is by Mills' ratio in \cref{appendix:mills-ratio-and-related-lemmas}. All other equalities are calculation of Gaussian integrals. 
    
    Using $\varphi(a) \cdot \varphi(b) = \Theta\big( e^{-\frac{a^2+b^2}{2}} \big)$, and plugging in \cref{lem:lower-bound-w-v}, for any fixed $\lambda,\rho$ we have:
    \begin{align*}
        \lrangle{v_\mathrm{relu}, v^\perp} = \Omega(1)\cdot \sqrt{\frac{\alpha\cdot \OPT + \eps}{\Phi(b)}} \cdot \frac{e^{- \frac{b^2}{2} (1-\lambda \rho)^2 }}{ b^2} = \Omega\big(\sqrt{\alpha \cdot \OPT + \eps} \big) \cdot \frac{ e^{- \frac{b^2}{2} \big((1-\lambda \rho)^2 - \frac{1}{2}\big) }}{|b|^{3/2}}
    \end{align*}
\end{proof}

\subsubsection{Bounding noise and showing progress}
In \cref{appendix:proof-lem-PGD-contribution-noise}, we similarly upper bound the contribution from label noise $y - \sigma(\lrangle{x,v}+b)$:
\begin{lemma}\label{lem:PGD-contribution-noise}
    Suppose \cref{assumption:standard} holds. Let $\lambda,\rho \in (0,1)$ be constants. For all unit vector $u$ such that $u\perp w$, and for all sufficiently negative $b$, we have:
    \begin{align}
        \lrangle{v_\mathrm{noise}, u} & = O(\sqrt{\OPT})\cdot \frac{e^{-\frac{b^2}{2} \big( -\rho^2 + \frac{1}{2}(\lambda + \rho \lambda^2 - \rho)^2 \big)}}{|b|^{1/2}}.\label{eq:PGD-noise-contribution}
    \end{align}
\end{lemma}

From these two lemmas, we can solve for the appropriate values of $\lambda$ and $\rho$, which allows 
$v_\mathrm{relu}$ to dominate $v_\mathrm{update}$, which indicates the direction of $v$. This proof is deferred to \cref{appendix:proof-lem-choosing-lambda-rho}.
\begin{lemma}\label{lem:choosing-lambda-rho}
    Suppose \cref{assumption:standard} holds where $b$ is sufficiently negative, and suppose $L(w_t, b) > \alpha\cdot \OPT$. Set $\lambda = 0.9$ and $\rho \in (0.3, 0.6)$. If $\lrangle{w_t, v} \geq \lambda$, the for all $u\perp w$ we have:
    \[ \frac{\abs{\lrangle{v_\mathrm{noise}, v^\perp}}}{\lrangle{v_\mathrm{relu}, v^\perp}} = e^{-\Omega(b^2)}. \]
\end{lemma}
Finally, we account for sampling error and finish the proof via the following lemma in \cref{appendix:PGD-making-progress}:
\begin{lemma}\label{lem:PGD:making-progress}
    Suppose \cref{assumption:standard} holds where $b$ is sufficiently negative, and suppose $L(w_t, b) > \alpha\cdot\OPT + \eps$ at some iteration $t$. Then, after an iteration of reweighted PGD with $\lambda = 0.9, \rho=0.5$, and $\eta = c_\eta \frac{\|w_t-v\|}{\|v_\mathrm{update}\|}$ for some $c_\eta \leq 0.1$, then:\bnote{double-check.}
    \[ \|w_{t+1} - v\|^2 \leq \big(1 - \Omega(c_\eta) \big) \|w_t - v\|^2. \]
\end{lemma}


A few extra comments about results in this subsection: 
\begin{enumerate}
    \item \cref{eq:solving-lambda-rho} does not have solution when $\lambda < \sqrt{\frac{1}{7} + \frac{2\sqrt{2}}{7}} \approx 0.74$, meaning that a warm start is necessary. This also agrees with the CSQ hardness, since reweighted PGD is a CSQ algorithm.  
    \item We also need $\rho > 0$ for the equation to be feasible, so reweighting is necessary. 
    \item It's not trivial how far we should shift the Gaussian mean along $w$. In particular, taking $\rho = 1$, then \cref{eq:solving-lambda-rho} is true only when $\lambda >1$, an infeasible solution. 
\end{enumerate}

\subsection{Thresholded PCA}\label{subsec:main-thresholded-PCA}
The final piece of our algorithmic result is to give a warm start via \textit{thresholding on $y$}. Our algorithm uses threshold $\tau = \frac{1}{|b|}$ and estimate direction $v$ via the top eigenvector of the following matrix:
\[ M = \EE[xx^\top\indicator{|y|\geq\tau}]. \]
The value $\tau = \Theta(1 / |b|)$ is the smallest threshold that reduces the noise to the necessary level (per \cref{lem:PCA-lem-1} and \cref{lem:PCA-lem-3}). We note that thresholding is related to \textit{trimmed/filtered PCA}~\citep{ChenM-COLT20-FilteredPCA, chen2020learning}, which uses a different matrix and works under the realizable setting. In our agnostic setting, the top eigenvector does not necessarily align with $v$, hence thresholded PCA only gives a coarse estimation and requires a different analysis. In this section we will show:
\begin{theorem}\label{thm:filteredPCA-main-thm}
    Suppose \cref{assumption:standard} holds. If bias $b$ satisfies $b \leq -\sqrt{\alpha / \log \alpha}$, then thresholded PCA with $\tau = \frac{1}{|b|}$ can with high probability find a unit vector $w$ such that:
    $$ \big| \lrangle{w, v} \big| \geq 1 - O\left( \frac{\log \alpha}{\alpha}\right), $$
    using $\poly(\log d, 1/\eps, \log 1/\delta)$ time and samples. 
\end{theorem}

Now we will present the proof, with details deferred to \cref{sec:appendix-filtered-PCA}. Imagine an adversary trying to perturb the top eigenvector of $M$. They can only do this in two ways:
\begin{enumerate}
    \item Generating noise in the (otherwise flat) region of the ReLU where $\lrangle{x,v}+b < 0$, increasing $M$'s magnitude in some direction $u\perp v$;
    \item Suppressing some of the $y$-value when $x$ has a high $v$-component, so the magnitude of $M$ along $v$ decreases.  
\end{enumerate}

We use three lemmas to show that these actions have limited effects when $\tau = \Theta(1 / |b|)$. For convenience, we define $M_0 = \EE_{(x,y)\sim \calD} \big[xx^\top \indicator{\abs{y} \geq \tau, \lrangle{v,x} + b < 0}\big]$, and $M_1 = M - M_0$. 
\subsubsection{Key lemmas}
First, we show that the ``flat'' region ($M_0$) contributes very little magnitude in any direction:
\begin{lemma}\label{lem:PCA-lem-1}
    Suppose \cref{assumption:standard} holds. For all sufficiently negative $b$, we have:
    \[ \|M_0\|_{\mathrm{op}} = O\lrparenth{\frac{\log \alpha}{\alpha}b^2 \Phi(b)}. \]
\end{lemma}
\begin{proof}
    Since the target ReLU is zero on the region $\{v^\top x + b < 0\}$, by Markov's inequality:
    $$ p := \PP [|y| \geq \tau, \lrangle{x, v} + b < 0] \leq \frac{\EE[y^2\indicator{|y|\geq \tau}]}{\tau^2} \leq \frac{\OPT}{\tau^2}. $$
    
    By \cref{lem:upper-bounding-OPT}, we have $\OPT = O\left( \frac{\Phi(b)}{\alpha b^2} \right)$. Plugging in the value of $\tau$, this means $p = O\lrparenth{ \frac{\Phi(b)}{\alpha}}$. For sufficiently negative $b$, $p$ would be small enough for us to apply Fact \ref{lem:bounding-x-squared-unlikely-event}. For all unit $u\in\mathbb{R}^d$,
    \begin{align*}
        u^\top M_0 u & = \EE[ \lrangle{x,u}^2 \indicator{|y|\geq \tau, \lrangle{x,w} + b < 0} ] \\
        & = O\Big( p\log \frac{1}{p} \Big) = O\lrparenth{\frac{\log \alpha}{\alpha}b^2 \Phi(b)}.
    \end{align*}
\end{proof}
Then, by the coordinate-wise independence property of standard Gaussian, we prove the following. 
\begin{lemma}\label{lem:PCA-lem-2}
    For all sufficiently negative $b$, and for any unit vector $u \perp v$, we have $u^\top M_0 u \leq \Phi(b)$. 
\end{lemma}

The last lemma shows that that the ReLU always has substantial contribution to $M$. The main observation is that, since $\OPT$ only has an $\frac{1}{\alpha}$-fraction of the ReLU's squared $L^2$ norm, with an $\OPT$ budget the adversary can only remove a small fraction of the ReLU's contribution from the calculation of $M$. The proofs of \cref{lem:PCA-lem-2} and \cref{lem:PCA-lem-3} can be found in \cref{sec:appendix-filtered-PCA}. 
\begin{lemma}\label{lem:PCA-lem-3}
    Suppose \cref{assumption:standard} holds. For all sufficiently negative $b$, $v^\top M_1 v = \Omega\big( b^2 \Phi(b)\big)$. 
\end{lemma}



%% file: main_hardness_results.tex
Our correlational statistical query (CSQ) hardness result follows an established template via a family $\mathcal G$ of functions with small \textit{pairwise correlation}~\citep{FeldmanGRVX13-planted-clique-CSQ-hardness-framework}. Our first step is to identify an appropriate family of functions $\calG$. 
Essentially, identifying one function from $\calG$ is hard under the CSQ model, and 
our goal is to reduce this task to agnostic learning a ReLU neuron. Notably, this proof only goes through for very negative $b$, as mentioned in \cref{sec:main-warm-up-and-prelim}. (More details in \cref{sec:appendix-hardness-results}.)  

Let $\|\cdot\|_\calN$ and $\lrangle{\cdot, \cdot}_\calN$ denote the $L^2$ norm and inner product with respect to the 1-dimensional standard Gaussian. Let $H_k$ be the $k$th (unnormalized) Hermite polynomial. The key lemma for our proof is the following:
\begin{lemma}\label{lem:CSQ-t-eps-unbounded}
    The following holds for all sufficiently small $\eps$. Let $g_\eps(x) = \sigma(x - b_\eps)$, where $b_\eps$ is chosen so that $\|g_\eps\|_\calN^2 = 3\eps$. 
    Let integer $t_\eps \in \NN$ be:
    \[ t_\eps := \max\left\{t\in\NN: \sum_{k=0}^t \lrangle{g_\eps, \frac{H_k}{\sqrt{k!}}}_\calN^2 \leq  
    \frac{\eps}{\alpha} \right\}, \]
    then, we have $t_\eps \to \infty$ as $\eps \to 0$. 
\end{lemma}

In words, as $b\to -\infty$, we can remove more and more lower-order Hermite components from the function $\sigma(x-b)$, while still having a constant fraction of its $L^2$ norm preserved. 

The fact that $g_\eps$ has \textit{no correlation} with lower-order Hermite polynomials allows us to construct the following family of functions with small pairwise correlation:
\begin{lemma}
    Let $\tilde{g}_\eps(x)$ be the functions $g_\eps(x)$ with its first $t_\eps$ Hermite components removed. Let $S$ be a set of vectors with $|\lrangle{u,v}| = O(d^{-\Omega(1)})$ for distinct $u,v\in S$, and $|S| = 2^{\Omega(d^c)}$. Consider the following family of functions:
    \[ \calG = \{g_\eps(\lrangle{x,v}):v\in S\}, \]
    then $\calG$ has low pairwise correlation. Specifically, for any distinct $u, v\in S$, we have:
    \[\begin{cases}
        \EE_{x\sim \calN(0, I_d)} [g_\eps(\lrangle{x,u})\cdot g_\eps(\lrangle{x,v})] \leq d^{-\Omega(t_\eps)}\cdot 3\eps,\\
        \EE_{x\sim \calN(0, I_d)} [g_\eps(\lrangle{x,v})^2] = 3\eps.
    \end{cases}\]
\end{lemma}

The low pairwise correlation and the large cardinality $|\calG|$ translates into high CSQ dimension, which leads to our desired conclusion: any CSQ algorithm that learns $g_\eps(\lrangle{x,v}) \in \calG$ to a non-trivial squared loss (better than the zero function) would make $2^{d^{\Omega(1)}}$ queries\footnote{In this context, we can make correlational queries in the form $\EE[yf(x)]$, where $x\sim\calN(0, I_d)$ and $y$ is labeled by some $g_\eps(\lrangle{x,v})$. }, or queries of tolerance $d^{-\Omega(t_\eps)}$. 

Finally, an agnostic learner for ReLU with $\alpha\cdot\OPT + \eps$ error can indeed learn any $g_\eps(\lrangle{x,v})$ up to nontrivial squared loss. Note that $\|g_\eps\|_\calN^2 = 3\eps$: the zero function incurs squared loss of $\|\tilde{g}_{t_\eps}\|_\calN^2 \geq (3 - \frac{1}{\alpha})\eps$, while the ReLU function $\sigma(\lrangle{x,v}+b_\eps)$ has squared loss $\|g_\eps - \tilde{g}_\eps\|_\calN^2 = \frac{\eps}{\alpha}$. 


%% file: appendix.tex
\section{Related works}\label{sec:appendix-related-works}
Here we give a more detailed survey of some of the most relevant results in algorithmic learning theory, relating to either \textit{arbitrary bias} or \textit{ReLU regression}. 

\textbf{Biased linear models.} 
Several works explicitly considered learning halfspaces with arbitrary bias, or \textit{general halfspaces}. \citet{DiakonikolasKS18a-biased-halfspace-nasty-noise} considers learning general halfspaces under the more general \textit{nasty noise} model. They achieve an optimal $O(\eps)$ error rate, which translates into constant-factor approximation for agnostic learning. Later, \citet{DiakonikolasKTZ22-biased-halfspace-online-GD} gave a faster constant-factor approximation in the agnostic setting. More recently, \citet{DiakonikolasKLZ24-testable-learning-biased-halfspace} proposed an $\Tilde{O}(\sqrt{\OPT})+\eps$ agnostic \textit{tester-learner} for general halfspaces. 

Going beyond the agnostic model, \citet{DiakonikolasKKT22-biased-halfspace-general-Massart} gives algorithm and hardness results for learning general halfspace under \textit{Massart noise}~\citep{Massart06}. \citet{DiakonikolasDKW23-biased-halfspace-RCN} considers learning general halfspaces under the weaker \textit{random classification noise (RCN)} model. Also on under RCN, \citet{DiakonikolasDKW23-information-conputation-biased-margin-halfspaces-RCN} makes no anti-concentration assumption on $x$ and therefore their results carry over to the biased case. 

We note that the techniques for classification do not carry over to the regression case. For instance, in the agnostic setting of classification problems, the ground truth classifier is correct $1-\OPT$ of the time. In regression problems, however, the $y$-value can be perturbed on all outcomes. 


\textbf{Single index models.} Learning a ReLU neuron, in the realizable setting, is a special case of the \textit{single index models} (SIMs)~\citep{KakadeKKS11-single-index-early}, where the $y$ value depends on a single, unknown direction $v$. One key difference of SIM is that the joint distribution of $\lrangle{x,v}$ and $y$ is often known, corresponding to the case where we know $\|v\|$ and $b$ before-hand, and the task is to statistically estimate the direction $v$. The CSQ and SQ hardness of this problem are completely characterized by the information exponent~\citep{DamianLS22-CSQ-hardness-polynomial-of-one-dimension} and the generative exponent~\citep{DamianPLB24-single-index-optimal-SQ-lower-bound}, and a SQ-CSQ separation appears as the link function (activation function) becomes more involved. 

Lately, several works \citep{GollakotaGKS23-SIM-agnostic-1st, ZarifisWDD24-SIM-agnostic-2nd, WangZDD2024-agnostic-gaussian-SIM} sought to extend the study of SIMs to agnostic case. Particularly, the recent independent work of \citet{WangZDD2024-agnostic-gaussian-SIM} also proposed a two-stage algorithm to give a constant approximation, for link functions that satisfy a set of assumptions. We note that the key difference between these regimes and ours is again due to our allowing for \textit{arbitrary bias}, which breaks all usual assumptions for SIMs such as $\EE_{z\sim\calN(0,1)}[\sigma(z)^2] = 1$. 

The task of realizably learning a single ReLU neuron, however, is indeed solved by earlier works on SIMs and isotonic regression~\citep{kakade2011efficient, kalai2009isotron}. A number of other works ~\citep{Tian17-learning-unbiased-relu-under-gaussian-using-gradient-1, Soltanolkotabi17-learning-unbiased-relu-under-gaussina-using-gradient-2, KalanSA19-learning-unbiased-relu-under-gaussina-using-gradient-3} also showed the effectiveness of various gradient methods in this relatively simple setting. 


\textbf{Agnostic learning of a single ReLU. } While a single ReLU neuron is relatively easy to learn to error $\OPT + \eps$ with the realizable assumption, one line of work~\citep{GoelKK19-relu-hardness-1, DiakonikolasKZ20-relu-hardness-2, GoelGK20-relu-hardness-3, DiakonikolasKPZ-COLT21-OptimalityOfPolyRegression, DiakonikolasKR23-cryptographic-hardness-Gaussian-ReLU} gave strong evidence that the same task takes quasi-polynomial time without this assumption. In other words, it is hard to substantially outperform polynomial regression~\citep{KalaiKMS05-L1-poly-regression} on this goal, in the agnostic setting. On the other hand, there exists PTAS with error $(1+\mu)\OPT + \eps$ which runs in time $\mathrm{poly}(d, \eps)\exp(1 / \mu)$\citep{DiakonikolasGKK20-agnostic-learning-unbiased-relu-constant-convex-surrogate, DiakonikolasKKT21-agnostic-proper-learning} for unbiased ReLUs.  

Further relaxing the objective, a number of $\poly(d/\eps)$-time algorithms are known. \cite{GoelKK19-relu-hardness-1} gave an $O(\OPT^{2/3})+\eps$ algorithm for unbiased ReLUs via reducing to learning \textit{homogeneous halfspace}. Some work applied GD on the convex \textit{matching loss} instead of the $L^2$ objective itself, achieving better time complexity under weaker distributional assumptions, for a class of activation functions including unbiased ReLU~\cite{DiakonikolasGKK20-agnostic-learning-unbiased-relu-constant-convex-surrogate, WangZDD23-agnostic-learning-unbiased-activations-constant-SGD-convex-surrogate}. Notably, we always need \textit{some} distributional assumptions for these results, as \cite{DiakonikolasKMR22-hard-to-constnat-approximate-relu-distribution-free} demonstrated a cryptographic hardness result for distribution-free constant-factor approximation. 

Another line of work proposes to apply GD on the $L^2$ objective itself. For unbiased ReLU, \cite{FreiCG20-agnostic-learning-unbiased-relu-OPT2/3} proved an $O(\OPT^{2/3})+\eps$ loss guarantee. Later, \cite{DiakonikolasKTZ22-agnostic-learning-unbiased-activations-GD} showed that GD actually achieves the best-possible $O(\OPT)+\eps$ error, for unbiased ReLUs. Their method also extends for ReLUs with a known \textit{positive} bias. The most relevant work to ours is that of \cite{AwasthiTV-ICLR23-ReLUGD}, where GD is shown to produce constant approximation even under moderately negative bias. However, as $b\to -\infty$, their approximation factor depends exponentially on $b$, which translates to a $(\OPT^{-\poly(\eps)} + \eps)$-approximation. All methods above are CSQ in nature, so they do not extend to our problem.

\section{Mills ratio and Gaussian integrals}
\label{appendix:mills-ratio-and-related-lemmas}
Let $\varphi,\Phi$ be the pdf and cdf of a one-dimensional Gaussian respectively. The \textit{Mills ratio} at $t$, for $t>0$, is defined as:
$$ m(t) = \frac{1 - \Phi(t)}{\varphi(t)}, $$
namely the ratio between Gaussian tail and Gaussian density. This ratio has the following asymptotic expansion ``around infinity" (see e.g. \cite{book:small2010expansions}):
$$ m(t) \sim \frac{1}{t} - \frac{1}{t^3} + \frac{1\cdot 3}{t^5} - \frac{1\cdot 3\cdot 5}{t^7} + \ldots, $$
where the error of every partial sum is bounded by the absolute value of the next term. 

We note that this series diverges for any fixed $t$, since the numerator grows like a factorial. On the other hand, at the limit $t\to\infty$, the more terms we have in a partial sum, the quicker the approximation error of this partial sum converges to zero.

Using Mills ratio, we now estimate the values of the following integrals:
\begin{itemize}
    \item The following is often used in thresholded PCA: As $b\to -\infty$, 
    \begin{align*}
        \int_{|b|}^\infty t^2 \,d\Phi(t) &= |b|\varphi(b) + \Phi(b) \\
        & = \big(1+o(1)\big) b^2\Phi(b).
    \end{align*}

    \item The next one is very useful in general, as it bounds the $L^2$ norm of a negatively biased ReLU nrueon: as $b\to -\infty$, 
    \begin{align*}
        \int_{|b|}^\infty (t+b)^2 \,d\Phi(t) & = (b^2 + 1)\Phi(b) - |b|\varphi(b) \\
        &  = \frac{2+o(1)}{b^2}\Phi(b). 
    \end{align*}

    \item The following does not need $b\to-\infty$ but is often used. For all $b$, we have:
    \[ \int_{-b}^\infty t(t+b)\,d\Phi(t) = \Phi(b). \]
    We note that this is a special case of \cref{lem:CSQ-explicit-Hermite-coefficient}. 

    \item
    This is used in proof for \cref{lem:PCA-lem-3} with $b-\tau$ in place of the value $b$ below. Let $h = \frac{c}{|b|}$ for some constant $c$, then:
    $$ \int_{|b| + h}^\infty (t + b)^2 \varphi(t)\,dt = (1+o(1))\cdot e^{-c}\cdot \left( 1 + c + \frac{c^2}{2} \right) \cdot \int_{|b|}^\infty (t+b)^2 \varphi(t)\,dt. $$
    
    Using the asymptotic expansion of Mills ratio:
    \begin{align*}
        & \int_{|b| + h}^\infty (t+b)^2 \varphi(t)\,dt = (b^2 + 1)\Phi\left( b - h\right) +  (b+h) \varphi\left( b - h \right) \\
        & = \varphi\left( b - h\right) \Bigg[ (b^2 + 1) \lrparenth{\frac{1}{|b - h|} - \frac{1}{ |b - h |^3}+ \frac{3+o(1)}{|b - h|^5} } + \underbrace{\left( b + h \right)}_{\frac{b^2 - h^2}{-|b-h|}} \Bigg] \\
        & = \frac{\varphi(b-h) }{|b-h|} \lrbracket{ (b^2+1)\left(1 - \frac{1}{ (b - h)^2}+ \frac{3+o(1)}{(b - h)^4}\right) - (b^2 - h^2) } \\
        & = \Phi(b - h) \lrbracket{ (b^2+1)\left(- \frac{1}{ (b - h)^2}+ \frac{3+o(1)}{(b - h)^4}\right) + 1 + h^2) } \\
        & = \Phi(b - h) \cdot \frac{-(b^2 + 1) + (1+h^2)(b-h)^2 + \frac{(3+o(1))(b^2 + 1)}{(b-h)^2}}{ (b - h)^2} \\
        & = \Phi(b - h) \cdot \frac{-1 - 2bh + h^2 + b^2h^2 -2bh^3 + h^4 + \frac{(3+o(1))(b^2 + 1)}{(b-h)^2}}{ (b - h)^2}.  \\
    \end{align*}
    
    Plugging in $h = \frac{c}{|b|} = o(1)$, we have
    \begin{align*}
        \int_{|b| + h}^\infty&  (t+b)^2 \varphi(t)\,dt \\
        & = \Phi(b - h) \cdot \frac{-1 + 2c + h^2 + c^2 +2ch^2 + h^4 + \frac{(3+o(1))(b^2 + 1)}{(b-h)^2}}{ (b - h)^2} \\
        & = \Phi(b - h)\cdot \frac{c^2 + 2c + 2 + o(1)}{b^2} \\
        & = (1+o(1))\cdot e^{-c}\cdot \left( 1 + c + \frac{c^2}{2} \right) \cdot \int_{|b|}^\infty  (t+b)^2\varphi(t)\,dt. 
    \end{align*}

\end{itemize}


\section{Omitted proofs from \cref{sec:main-warm-up-and-prelim}}\label{appendix:proof-prelims}
Recall that $\hat{\calD}$ is the distribution over $(x,\hat{y})$ where $\hat{y} = \frac{y}{\|v\|}$. Suppose the best-fitting ReLU for $\calD$ is some $\sigma(\lrangle{x,v} + b)$ where $v$ is not necessarily a unit vector. 
\begin{proposition}[Same as \cref{prop:prelim:scaling-v}]
    In the above setting, the optimal ReLU for $\hat{\calD}$ is $\sigma(\lrangle{x, \hat{v}} + \hat{b})$ where $\|\hat{v}\| = v / \|v\|$ is a unit vector and $\hat b = b/\|v\|$, and it has loss $\widehat{\OPT} = \hat{L}(v,b) = \OPT / \|v\|^2$. 
    
    Moreover, suppose parameters $w,b_w$ incur loss $\hat{L}(w, b_w) \leq \alpha\cdot\widehat{\OPT} + \eps$ on the normalized problem $\hat{\calD}$. Then, the pair $(\|v\|w, \|v\|b_w)$ would incur loss $L(\|v\|w, \|v\|b_w) \leq \alpha\cdot \OPT + \|v\|^2\eps$ on the original problem $\calD$. 
\end{proposition}
\begin{proof}
    The first fact follows from the fact that the ReLU function $\sigma$ is homogeneous:
    \[ \sigma(\lrangle{x,w}+b_w) - y = \|v\|\cdot \left( \sigma \Big(\lrangle{x,\frac{w}{\|v\|}} + \frac{b_w}{\|v\|} \Big) - \hat{y}\right),\text{ for all choices of }w\in \RR,b_w\in \RR. \]
    
    The second is a result of scaling everything by $1/\|v\|$, except the additive $\eps$:
    \begin{align*}
        \frac{1}{2}\EE_{\calD}[(\|v\| h(x) - y)^2] & = \|v\|^2 \cdot \frac{1}{2}\EE[ (h(x) - \hat y)^2 ] \\
        & \leq \|v\|^2 \big(\frac{\alpha}{2} \cdot \EE_{\hat \calD}[(\sigma(\lrangle{x, \hat v} + \hat b) - \hat y)^2] + \eps\big) \\
        & = \alpha \cdot \frac{1}{2} \EE[(\sigma(\lrangle{x, v} + b) - y)^2] + \|v\|^2 \eps \\
        & = \alpha\cdot\OPT + \|v\|^2 \eps.
    \end{align*}
\end{proof}

\begin{proposition}[Same as \cref{prop:prelims:search-v-b}]
    Let $\hat v$ be the unit vector in $v$'s direction. Let $\delta_v = \beta - \|v\|$ and $\delta_b = \gamma - b$ denote the additive errors of our parameter estimations. Then if $|\delta_v|, |\delta_b| \leq 0.1\sqrt{\eps}$, we have:
    \[ \alpha \cdot L(\beta \hat v, \gamma) \leq O(\alpha)\cdot L(v, b) + O(\eps). \]
\end{proposition}
\begin{proof}
    We will repeatedly apply the elementary inequality $(a+b)^2 \leq 2a^2 + 2b^2$. First, we compare the loss of the estimated optimal ReLU with $\OPT$:
    \begin{align*}
        \EE\big[\big(\sigma(\lrangle{x, v+\delta_v \hat v} & + b + \delta_b) - y\big)^2 \big] \\
        & \leq 2\EE\big[\big(\sigma(\lrangle{x, v+\delta_v \hat v} + b + \delta_b) - \sigma(\lrangle{x, v} + b)\big)^2 \big] + 2 \OPT \\
        & \leq 4\EE\big[\big(\sigma(\lrangle{x, v+\delta_v \hat v} + b + \delta_b) - \sigma(\lrangle{x, v} + b + \delta_b)\big)^2 \big] \\
        & \;\; + 4\EE\big[\big(\sigma(\lrangle{x, v} + b + \delta_b) - \sigma(\lrangle{x, v} + b) \big)^2 \big] + 2\OPT.
    \end{align*}

    Since $\sigma$ is 1-Lipschitz, we can remove it, cancel the terms before squaring, and get an upper bound:
    \begin{align*}
        \EE\big[\big(\sigma(\lrangle{x, v+\delta_v \hat v} + b + \delta_b) - y\big)^2 \big] & \leq 4\EE [\big(\lrangle{x, \delta_v \hat v})^2 ] + 4 \delta_b^2 + 2\OPT.
    \end{align*}
    Note that the first term on the right is just $4\delta_v^2$. Taking $|\delta_v| = |\delta_b| \leq 0.1\sqrt{\eps}$, and the proof is finished. 
\end{proof}

\begin{lemma}[Same as Lemma \ref{lem:upper-bounding-OPT}]
    Suppose $\alpha \geq 3$. Let $\sigma(\lrangle{x,v}+b)$ be the optimal ReLU with loss $\OPT$, where $\|v\|=1$ and $b$ sufficiently negative. If the zero function incurs loss at least $\alpha\cdot\OPT + \eps$, then we have
    \[ \OPT < \frac{3\Phi(b)}{\alpha b^2}\text{, and }\eps < \frac{3\Phi(b)}{b^2}. \]
\end{lemma}
\begin{proof}

    By assumption we know that $\frac{1}{2}\EE[y^2] > \alpha \OPT$. Using Gaussian integrals and Mills ratio, we have:
    \begin{align*}
        \alpha \cdot \OPT + \eps < \frac{1}{2}\EE[y^2 ]  & \leq \EE\big[ \big( y - \sigma(\lrangle{x,v} + b) \big)^2 \big] + \EE[ \sigma(\lrangle{x,v} + b)^2 ] \\
        & = 2\OPT + \big(2 + o(1)\big) \frac{\Phi(b)}{b^2}.
    \end{align*}
    Rearranging the terms:
    \begin{align*}
        \OPT + \frac{\eps}{\alpha - 2}& < \frac{\big(2 + o(1)\big)\Phi(b)}{(\alpha - 2)b^2} \\
        & = \frac{3 \Phi(b)}{\alpha b^2}.
    \end{align*}
    Since $\alpha \geq 3$, both claims follow.  
\end{proof}

%% file: appendix_reweighted_PGD.tex


In this section we provide the missing details from \cref{subsec:main-reweighted-PGD}. First, we introduce some notations used in this section. We will focus heavily on the plane spanned by vectors $w_t$ and $v$. As in \cref{subsec:main-reweighted-PGD}, we use $x_w = \lrangle{x,w_t}$ and $x_\perp = \langle x, v^\perp\rangle$, where:
\[ v^\perp = \frac{v - \lrangle{v,w_t}w_t}{\|v - \lrangle{v,w_t}w_t\|}. \]

We note that every $x$ on this plane can be written as $x = x_w\cdot w_t + x_\perp\cdot v^\perp$ by orthogonality. 

\subsection{Re-centering the Gaussian covariates}\label{subsec:PGD-weighted-PGD}
As noted in \cref{sec:main-warm-up-and-prelim}, GD modifies the direction of $w_i$ by gearing it towards the direction of the \textit{conditioned correlation} between $x$ and $y$, $\EE[x y\cdot \indicator{\lrangle{x,w_t} \geq -b}]$,
to learn about the direction $v^\perp$. One reason for its failure is that the condition function is too coarse to have the following ideal properties:
\begin{enumerate}
    \item It should attempt to ignore the regions where $\lrangle{x,v} \leq -b$, since those are where the optimal ReLU neuron is zero and not informative. Conversely, it would put more weight on regions with large $\lrangle{x,v}$. 
    \item It should do so using the current estimated direction $w_i$, without knowing the true direction $v$. 
\end{enumerate}

In the rest of this subsection we will describe our approach, which applies both a \textit{weight function} $f_{t}(x)$ and a \textit{condition function} $\mathbbm{1} \{\lrangle{x,w_t} \geq \Tilde{b}\}$, in order to achieve the objectives above. Instead of $\EE[x y\cdot \indicator{x_w \geq |b|}]$, on each iteration we make a ``gradient'' update using the vector
\[ v_{\mathrm{update}} =  \big(I_d - w_t w_t^\top\big) \EE [f_{t}(x)\cdot x \cdot y \cdot \mathbbm{1} \{ x_w \geq  \Tilde{b} \} ], \]
where the two functions are defined via parameters $\rho,\lambda\in (0,1)$, to be determined in \cref{appendix:proof-lem-PGD-contribution-noise}:
\[ \begin{cases}
    f_t(x) = \exp\big( \rho|b|x_w - \frac{1}{2}\rho^2 b^2 \big), \text{ and }\\
    \Tilde{b} = (\rho + \lambda - \lambda^2\rho)|b|.
\end{cases}\]

Now we explain the choice of these values. Weight function $f_t(x)$ is the relative density that shifts the $x$-marginal from standard Gaussian $\calN(0, I_d)$ to $\calN(\rho |b| w_t, I_d)$. Parameter $\rho$ controls the amount of shift, relative to $b$. 

Parameter $\lambda$ can be understood as a lower bound for $\lrangle{w_t, v}$. The condition function $\Tilde{b} = (\rho + \lambda - \lambda^2\rho)|b|$ is chosen to minimize the influence of noise while maintaining most of contribution from ReLU. Specifically, if $\lrangle{w_t, v} = \lambda$, then the value $\Tilde{b}$ is precisely the $x_w$-coordinate of the point in region $\{x: \lrangle{x,v} \geq |b|\}$ that is the \textit{closest to the new Gaussian center $\rho|b|w_t$}. In other words, it cuts through the region where ReLU makes a positive contribution, in a way that at least half of the contribution from ReLU are preserved. 



\subsection{Contribution from optimal ReLU}\label{subsec:PGD-contribution-relu}
The following is the technical lemma used in the proof of \cref{lem:main-PGD-lower-bound-ReLU}. In short, we lower-bound the contribution from ReLU on every ``horizontal slice'' (as in \cref{fig:intro:gradient-when-bias-small-vs-large}) of fixed $x_w$. Let $\Lambda = \lrangle{w_t,v}$. 
\begin{lemma}\label{lem:technical-relizable-update}
    Suppose $\rho,\lambda\in (0,1)$ satisfy $\rho\lambda > \frac{1}{2}$, and assume $b < 0$. If $\Lambda \geq \lambda$, then on event $E = \{x_w \geq (\lambda-\lambda^2\rho+\rho)|b|\}$, for all $x_\perp \geq 0$ we have:
    \begin{align}
        \sigma\big(\langle x_w w + x_\perp v^\perp,v\rangle+b\big) - & \sigma\big(\langle x_w w - x_\perp v^\perp,v\rangle+b\big) \nonumber\\
        & \geq \sqrt{1 - \Lambda ^2}\cdot \sigma\Big(x_\perp - \sqrt{1-\lambda^2}(1-\lambda\rho)|b| \Big).
    \end{align}
\end{lemma}
\begin{proof}
    Since $v = \Lambda w + \sqrt{1-\Lambda^2}v^\perp$, we have:
    \begin{align*}
        \sigma\big(\langle x_w w + x_\perp v^\perp,v\rangle+b\big) - & \sigma\big(\langle x_w w - x_\perp v^\perp,v\rangle+b\big) \nonumber\\
        & = \sigma\Big(\Lambda x_w + \sqrt{1-\Lambda^2} x_\perp - |b| \Big) - \sigma\Big(\Lambda x_w - \sqrt{1-\Lambda^2} x_\perp - |b|\Big) \\
        & \geq \sigma\Big(\Lambda x_w + \sqrt{1-\Lambda^2} x_\perp - |b| \Big) - \sigma(\Lambda x_w - |b|)
    \end{align*}

\textbf{Case 1}, when $\Lambda x_w < |b|$. In this case the second term is zero, and the first term is:
\begin{align*}
    \sigma\Big(\Lambda x_w + \sqrt{1-\Lambda^2} x_\perp - |b| \Big) &\geq \sigma\Big(\sqrt{1-\Lambda^2} x_\perp + \Lambda (\lambda - \lambda^2\rho + \rho)|b| - |b| \Big) \\
    & \geq \sigma\Big(\sqrt{1-\Lambda^2} x_\perp + \Lambda (\Lambda - \Lambda^2\rho + \rho)|b| - |b| \Big),
\end{align*}
since $\lambda - \lambda^2\rho + \rho$ has derivative $1 - 2\lambda \rho < 0$. We can thus further bound it by:
\begin{align*}
    \sigma\Big(\sqrt{1-\Lambda^2} x_\perp + \Lambda (\Lambda - \Lambda^2\rho + \rho)|b| - |b| \Big) & =  \sigma\Big(\sqrt{1-\Lambda^2} x_\perp + (\Lambda^2 - \Lambda^3\rho + \Lambda\rho - 1)|b| \Big) \\
    & = \sigma\Big(\sqrt{1-\Lambda^2} x_\perp - (1 - \Lambda^2)(1-\Lambda\rho)|b| \Big) \\
    & \geq \sigma\Big(\sqrt{1-\Lambda^2} x_\perp - \sqrt{1-\Lambda^2}\sqrt{1-\lambda^2}(1-\lambda\rho)|b| \Big) \\ 
    & = \sqrt{1 - \Lambda ^2}\cdot \sigma\Big(x_\perp - \sqrt{1-\lambda^2}(1-\lambda\rho)|b| \Big),
\end{align*}
as desired. 

\textbf{Case 2}, when $\Lambda x_w < |b|$. In this case we have:
\begin{align*}
    \sigma\Big(\Lambda x_w + \sqrt{1-\Lambda^2} x_\perp - |b| \Big) - \sigma(\Lambda x_w - |b|) & = \Lambda x_w + \sqrt{1-\Lambda^2} x_\perp - |b| - (\Lambda x_w - |b|) \\
    &= \sqrt{1+\Lambda}x_\perp \\
    & \geq \sqrt{1 - \Lambda ^2}\cdot \sigma\Big(x_\perp - \sqrt{1-\lambda^2}(1-\lambda\rho)|b| \Big).
\end{align*}
\end{proof}

\subsection{Proof for Lemma \ref{lem:lower-bound-w-v}}\label{appendix:proof-lem-lower-bound-w-v}
\begin{lemma}[same as Lemma \ref{lem:lower-bound-w-v}]
    Suppose \cref{assumption:standard} holds. For any unit vector $w$, if $L(w, b) > \alpha\cdot \OPT + \eps$, then for all sufficiently negative $b$:
    \[ \|w-v\|^2 \geq \Omega \left( \frac{\alpha\cdot \OPT + \eps}{\Phi(b)} \right). \]
\end{lemma}
\begin{proof}
   Let $F(w,b)$ be the realizable loss $\EE[(\sigma(\lrangle{x,w}+b) - \sigma(\lrangle{x,v}+b))^2]$. By an elementary inequality we have $L(w, b) \leq 2\OPT + F(w,b)$, which implies $F(w) > 0.4\alpha\cdot \OPT + \eps$ for large $\alpha$. Meanwhile, $F(w)$ is upper bounded by $\|w-v\|$ by the following:
   \begin{align*}
        F(w, b) & = \EE[(\sigma(\lrangle{x,w} + b) - \sigma(\lrangle{x,v} + b))^2]\\
        & \leq \EE[(\lrangle{x,w} - \lrangle{x,v})^2 \indicator{\lrangle{x,w} + b \geq 0\text{ or }\lrangle{x,v} + b \geq 0}] \\
        & \leq 2 \EE[\lrangle{x, w-v}^2 \indicator{\lrangle{x,w} \geq |b|}] \\
        & = 2\int_{|b|}^\infty \int_{-\infty}^\infty \lrangle{x_w w + x_\perp v^\perp, w-v}^2 \,d\Phi(x_\perp)\,d\Phi(x_w) \\
        & = 2\int_{|b|}^\infty \int_{-\infty}^\infty \left( x_w \frac{\|w-v\|^2}{2} - x_\perp \frac{\|w-v\|\|w+v\|}{2} \right)^2 \,d\Phi(x_\perp)\,d\Phi(x_w) \\
        & \leq \int_{|b|}^\infty  \int_{-\infty}^\infty \big( x_w^2 \|w-v\|^4 - x_\perp^2 \|w-v\|^2\big) \,d\Phi(x_\perp)\,d\Phi(x_w) \\
        & = \int_{|b|}^\infty \big(x_w^2 \|w-v\|^4 + \|w-v\|^2\big) \,d\Phi(x_w) \\
        & \leq \big( 1.1\|w-v\|^4 b^2 + \|w-v\|^2 \big)\Phi(b),
    \end{align*}
    for sufficiently negative $b$. 

    If $\|w-v\|\leq \frac{1}{|b|}$, then $F(w, b) \leq 2.1\|w-v\|^2\Phi(b)$, and we have $\|w-v\| > \sqrt{\frac{0.4 \alpha\cdot \OPT + \eps}{2.1\Phi(b)}}$ as desired. 

    On the other hand, if $\|w-v\| > \frac{1}{|b|}$, then $\|w-v\|^2\Phi(b) \geq \frac{\Phi(b)}{b^2} \geq \frac{\alpha\cdot \OPT + \eps}{6}$ by \cref{lem:upper-bounding-OPT}, and the proof is finished. 
\end{proof}

\subsection{Proof of Lemma \ref{lem:PGD-contribution-noise}: contribution from noise}\label{appendix:proof-lem-PGD-contribution-noise}

\begin{lemma}[same as Lemma \ref{lem:PGD-contribution-noise}]
    Suppose \cref{assumption:standard} holds. Let $\lambda,\rho \in (0,1)$ be constants. For all unit vector $u$ such that $u\perp w$, and for all sufficiently negative $b$, we have:
    \begin{align}
        \lrangle{v_\mathrm{noise}, u} & = O(\sqrt{\OPT})\cdot \frac{e^{-\frac{b^2}{2} \big( -\rho^2 + \frac{1}{2}(\lambda + \rho \lambda^2 - \rho)^2 \big)}}{|b|^{1/2}}.
    \end{align}
\end{lemma}
\begin{proof}
    Using Cauchy-Schwarz, we can bound $\lrangle{v_\mathrm{noise}, u}$ by:
    \begin{align*}
        \underset{\calD}{ \EE }[f(x)\cdot \lrangle{x,u} & \cdot (y - \sigma(\lrangle{x,v}+b))\indicator{E}] \\
        & \leq \sqrt{ \EE [(y-\sigma(\lrangle{x,v}+b))^2]} \sqrt{\EE [f(x)^2 \indicator{x_w \geq (\rho + \lambda - \lambda^2\rho)|b|}]}\\
        & = \sqrt{2\OPT} \cdot \sqrt{ \EE [f(x)^2 \indicator{x_w \geq (\rho + \lambda - \lambda^2\rho)|b|}]} 
    \end{align*}
    The second term can be explicitly written as an integral:
    \begin{align*}
         \EE[f(x)^2\indicator{x_w \geq (\rho + \lambda - \lambda^2\rho)|b|}] & = \int_{(\rho+\lambda+\rho\lambda^2)|b|}^\infty \varphi(x_w) \cdot \exp(2\rho|b|x_w-\rho^2b^2) \,dx_w \\
         &= \frac{1}{\sqrt{2\pi}}\int_{(\rho+\lambda+\rho\lambda^2)|b|}^\infty  \exp\left(-\frac{1}{2}x_2^2 + 2\rho|b|x_w - 2\rho^2b^2 +\rho^2b^2 \right) \,dx_w \\
         &= \exp(\rho^2b^2)\cdot \underset{x_w \sim \calN(2\rho |b|, I_d)}{\PP} \big[x_w \geq (\lambda+\rho\lambda^2 + \rho)|b| \big]\\
         &= \exp(\rho^2b^2)\cdot \Phi\big( (\lambda+\rho\lambda^2-\rho)b \big). 
    \end{align*}
    Plugging this back in, we have:
    \begin{align*}
        \lrangle{v_\mathrm{noise}, u} &\leq \sqrt{2\OPT} \cdot \sqrt{\exp(\rho^2b^2)\cdot \Phi\big( (\lambda+\rho\lambda^2-\rho)b \big)} \\
        &= O\big(\sqrt{\OPT}\big) \cdot \frac{1}{\varphi(\rho b)} \cdot \sqrt{\Phi\big( (\lambda+\rho\lambda^2-\rho)b \big)} \\
        &= O\big(\sqrt{\OPT}\big) \cdot \frac{1}{\varphi(\rho b)} \cdot \sqrt{\frac{\varphi\big( (\lambda+\rho\lambda^2-\rho)b \big)}{(\lambda + \rho\lambda^2 - \rho)|b|}}  \\ 
        &= O\big(\sqrt{\OPT}\big) \cdot \frac{e^{-\frac{b^2}{2}\big( - \rho^2 + \frac{1}{2}(\lambda + \rho \lambda^2 - \rho)^2 \big) }}{|b|^{1/2}} , 
    \end{align*}
    as desired. 
\end{proof}
\subsection{Proof for Lemma \ref{lem:choosing-lambda-rho}}\label{appendix:proof-lem-choosing-lambda-rho}
Now, we combine the previous lemmas to show that the direction of $v_\mathrm{update}$ is dominated by the contribution from ReLU, rather than noise. We will also determine the choice for parameters $\rho$ and $\lambda$. 
\begin{lemma}[same as Lemma \ref{lem:choosing-lambda-rho}]
    Suppose \cref{assumption:standard} holds where $b$ is sufficiently negative, and suppose $L(w_t, b) > \alpha\cdot \OPT$. Set $\lambda = 0.9$ and $\rho \in (0.3, 0.6)$. If $\lrangle{w_t, v} \geq \lambda$, the for all $u\perp w$ we have:
    \[ \frac{\abs{\lrangle{v_\mathrm{noise}, v^\perp}}}{\lrangle{v_\mathrm{relu}, v^\perp}} = e^{-\Omega(b^2)}. \]
\end{lemma}
\begin{proof}
    To compare \cref{eq:PGD-realizable-contribution} and \cref{eq:PGD-noise-contribution}, we first consider the terms involving $e^{-\frac{1}{b^2}}$. To ensure $\lrangle{v_\mathrm{noise}, v^\perp}$ dominates, we want it to have a \textit{smaller} coefficient inside $\exp \big(-\frac{b^2}{2}\big)$, namely:
\begin{equation}\label{eq:solving-lambda-rho}
    (1-\rho\lambda)^2 - \frac{1}{2} - \bigg( -\rho^2 + \frac{1}{2}(\lambda + \rho\lambda^2 - \rho)^2\bigg) \leq -\Omega(1).
\end{equation}

This is true when $\lambda \geq 0.9$ and $\rho \in [0.3, 0.6]$, in which case the right hand side is less than $-0.01$. Consequently, for any unit vector $u$ such that $u\perp w$,
\begin{align*}
    \frac{\abs{\lrangle{v_\mathrm{noise}, u}}}{\lrangle{v_\mathrm{relu}, v^\perp}} & = \frac{O(1)\cdot e^{-\frac{b^2}{2}\big( - \rho^2 + \frac{1}{2}(\lambda + \rho \lambda^2 - \rho)^2 \big) } / |b|^{1/2}}{ \Omega(\sqrt{\alpha}) \cdot e^{- \frac{b^2}{2} \big((1-\lambda \rho)^2 - \frac{1}{2}\big) } / |b|^{3/2} } \\
    & \leq \frac{O(1)\cdot e^{-\frac{0.16\cdot b^2}{2}} \cdot |b|}{ \Omega(\sqrt{\alpha}) } \\
    & = e^{-\Omega(b^2)}. 
\end{align*}
\end{proof}


\subsection{Making progress on each update}\label{appendix:PGD-making-progress}
Recall that our algorithm makes update by setting $w_{t+1} := \frac{w_t + \eta \hat v_\mathrm{update}}{\|w_t + \eta \hat v_\mathrm{update}\|}$, where $\hat v_\mathrm{update}$ is the estimation of $v_\mathrm{update}$ from $n$ new samples. In this subsection we will prove that, with appropriate values of $\eta$ and $n$, $w_t$ gets provably close to $v$ on each iteration with high probability. 

First we upper bound on magnitude $\|v_\mathrm{update}\|$:
\begin{lemma}
    Suppose \cref{assumption:standard} holds with $b$ sufficiently negative, and suppose $L(w_t, b) > \alpha\cdot \OPT$. For any $\rho, \lambda\in (0,1)$, if $\lrangle{w_t,v} \geq \lambda$, then:
    \[ \|v_\mathrm{update}\| = O(\|w_t - v\|)\cdot \Phi\big((1-\rho) b\big) \leq \poly(\eps). \]
\end{lemma}
\begin{proof}
    \begin{align*}
        \|v_\mathrm{update}\| & = \max_{\text{unit }u: u\perp w_t} \lrangle{v_\mathrm{update}, u} \\
        & \leq \lrangle{v_{\mathrm{relu}}, v^\perp} + \max_{\text{unit }u: u\perp w} \lrangle{v_\mathrm{noise}, u} \\
        & = \big(1 + o(1) \big)\lrangle{v_{\mathrm{relu}}, v^\perp},
    \end{align*}
    where the small $o$ is taken as $b\to -\infty$. It therefore suffices for us to upper bound $\lrangle{v_{\mathrm{relu}}, v^\perp}$. We will take $f_t(x)$ into account by considering the shifted distribution: let $\calD'$ be the modified distribution of $(x,y)$, with the $x$-marginal being $\calN(\rho |b|w_t, I_d)$. Then we have:
    \begin{align*}
        \lrangle{v_\mathrm{relu}, v^\perp} & = \underset{\calD}{ \EE } \Big[f_t(x)\cdot x_\perp \cdot \sigma(\lrangle{x,v} + b)\cdot \big(1 - \indicator{\lrangle{x,w_t} < (\rho+\lambda-\lambda^2\rho)|b|} \big)\Big]  \\
        & = \underset{\calD'}{ \EE } [x_\perp \cdot \sigma(\lrangle{x,v}+b)] - \underset{\calD'}{ \EE } [ x_\perp \cdot \sigma(\lrangle{x,v}+b) \indicator{\lrangle{x,w_t} < (\rho+\lambda-\lambda^2\rho)|b|} ].
    \end{align*}
    
    Note that the second term is nonnegative: for any fixed $x_w$, the value of $\sigma(\lrangle{x,v}+b)$ always grows with $x_\perp$, hence they have positive correlation. It now suffices to upper bound the first term:
    \begin{align}
        \underset{\calD'}{ \EE } [\langle x, v^\perp \rangle \cdot \sigma(\lrangle{x,v} + b)] & \leq \underset{\calD'}{ \EE } [\langle x-\rho|b|w_t, v^\perp \rangle \cdot \sigma(\lrangle{x-\rho|b|w_t,v} - (1 - \rho)|b|)] \nonumber \\
        & = \underset{\calD}{ \EE } [\langle x, v^\perp \rangle \cdot \sigma(\lrangle{x,v} - (1 - \rho)|b|)] \nonumber \\
        & = \langle v, v^\perp \rangle \Phi\big( (1-\rho)|b| \big). \nonumber
    \end{align}
    The first inequality uses $w_t\perp v^\perp$ and $\lrangle{w_t, v} \leq 1$. The second inequality applies change of variable, taking $x$ to be the previous $x - \rho|b|w_t$, since both random variables have distribution $\calN(0, I_d)$. The proof is finished by noting $\langle v, v^\perp \rangle = \sqrt{1 - \lrangle{w_t ,v}^2} = \Theta(\sqrt{1 - \lrangle{w_t , v}}) = \Theta(\|w_t - v\|)$. 
\end{proof}

\bnote{New:} Now we bound the number of fresh samples we need on each iteration of reweighted PGD. First, assuming $y$ is bounded, we have the following sample complexity bound:
\begin{lemma}\label{lem:PGD-finite-sample-v-update}
    Suppose $|y| \leq B$ almost surely. For all sufficiently negative $b$, if $\hat v_{\mathrm{update}}$ is calculated using $m = \poly\left( d, \frac{1}{\delta}, \frac{1}{\eps}, B \right)$ new samples, with probability $\geq 1 - \delta$ we have:
    \[ \langle \hat v_\mathrm{update}, v^{\perp}\rangle \geq 0.9 \|\hat v_{\mathrm{update}}\|, \]
    on any iteration $t$. 
\end{lemma}
\begin{proof}
    We can bound the LHS and RHS by:
    $$\langle \hat v_{\mathrm{update}}, v^\perp\rangle \geq \langle v_{\mathrm{update}}, v^\perp\rangle - \|\hat v_{\mathrm{update}} - v_{\mathrm{update}}\|,\text{ and }$$
    \[ 0.9\| \hat v_{\mathrm{update}} \| \leq 0.9(\|v_{\mathrm{update}}\| + \|\hat  v_{\mathrm{update}} - v_{\mathrm{update}}\|). \]
    Hence, it suffices to show that:
    \[ \langle v_{\mathrm{update}}, v^\perp\rangle - \|\hat v_{\mathrm{update}} - v_{\mathrm{update}}\| \geq 0.9(\|v_{\mathrm{update}}\| + \|\hat v_{\mathrm{update}} - v_{\mathrm{update}}\|). \]

    By Lemma \ref{lem:choosing-lambda-rho}, we know that $\langle v_{\mathrm{update}}, v^\perp \rangle \geq 0.95 \|v_{\mathrm{update}}\|$ as $b\to -\infty$, since $v_{\mathrm{relu}}$ dominates $v_{\mathrm{noise}}$. Now we want to show that, for all sufficiently negative $b$:
    \[ 0.05 \| v_{\mathrm{update}}, v^\perp \| \geq 1.9 \|\hat v_{\mathrm{update}} - v_{\mathrm{update}}\|. \]

    To lower bound the LHS we use Lemma \ref{lem:main-PGD-lower-bound-ReLU}, which states:
    \begin{align*}
        \|v_{\mathrm{update}}\| &= (1+o(1)) \| v_{\mathrm{relu}}\|\\
        & \geq \Omega\big(\sqrt{\alpha\cdot \OPT + \eps}\big) \cdot e^{-O(b^2)} \\
        & = \poly(\eps)\text{ by Lemma \ref{lem:upper-bounding-OPT}.}
    \end{align*}

    Meanwhile, to upper bound the RHS, we will use the assumption that $y$ is almost surely bounded, and apply multidimensional Chebyshev's inequality. 

    Let $\mu$ and $\Sigma$ denote the mean and covariance matrix of the vector $f_t(x) \cdot x\cdot y\cdot\indicator{x_w \geq (\rho+\lambda - \lambda^2\rho)|b|}$. For any unit vector $u$, we have:
    \begin{align*}
        u^\top \Sigma u & \leq \EE[ f_t(x)^2 \lrangle{x,u}^2 y^2 ] \\
        & \leq \big( \EE[y^4] \big)^{1/2} \big( \EE[f_t(x)^8] \big)^{1/4} \big( \EE[\lrangle{x,u}^8] \big)^{1/4} \\
        & = O(B^2) \cdot \left( \int_{-\infty}^\infty \exp(8\rho|b|x_w- 4\rho^2 b^2)\varphi(x_w)\,dx\right)^{1/4} \\ 
        & = O\big( B^2 e^{3\rho^2b^2} \big),
    \end{align*}
    where $\rho \in (0.3, 0.6)$ as before. Suppose the empirical estimate $\hat v_{\mathrm{update}}$ is calculated using $m$ samples, then the random variable $\hat v_\mathrm{update}$ has covariance $\frac{1}{m}\Sigma$. By multidimensional Chebyshev's inequality, for all $s>0$ we have:
    \begin{align*}
        \PP \bigg[ & \sqrt{(\hat v_\mathrm{update} -  v_{\mathrm{update}})^\top (\Sigma/m)^{-1} (\hat v_\mathrm{update} - v_{\mathrm{update}})} \geq s \bigg] \\
        & = \PP \left[ \|\hat v_\mathrm{update} - v_{\mathrm{update}}\|  \geq \frac{sBe^{1.5\rho^2b^2}}{\sqrt{m}} \right] \leq \frac{d}{s^2},
    \end{align*}
    To make this at most $\delta$, we take $s = \sqrt{\frac{d}{\delta}}$. The inequality now becomes:
    \[ \PP \left[ \|\hat v_\mathrm{update} - v_{\mathrm{update}}\| \geq \frac{\sqrt{d}Be^{1.5\rho^2b^2}}{\sqrt{m\delta}} \right] \leq \delta. \]

    Therefore, to conclude, we must set $m$ such that $\frac{\sqrt{d}sBe^{1.5\rho^2b^2}}{\sqrt{m}\delta}$ is at most some $\mathrm{poly}(\eps)$. A polynomial number of samples is sufficient:
    $$ m = \frac{O\left( \sqrt{d}s^2 B^{2}e^{3\rho^2b^2} \right) }{\sqrt{\delta} \cdot \mathrm{poly}(\eps)}  = \poly\left( d, \frac{1}{\delta}, \frac{1}{\eps}, B\right). $$
\end{proof}

    Now, if $\calD$ has unbounded $y$, we can slightly modify the samples so that we can more efficiently sample the same desired direction. Let $B_x(d, m, \eps, \delta)$ be the value with the following two properties: 
    \begin{enumerate}
        \item With probability at least $1-\delta$, all $m$ fresh samples $\{(x_i, y_i)\}_{i=1}^m$ from $\calD$ will have $\|x_i\| \leq B_x(d, m, \eps, \delta)$. 
        \item On the event $\|x_i\| \leq B_x(d, m, \eps, \delta)$, most of the value $\langle v_{\mathrm{relu}}, v^\perp\rangle$ should be kept (c.f. \cref{subsec:main-reweighted-PGD}):
        \begin{align*}
            \big | \EE[f_t(x)\cdot x_\perp\cdot \sigma(\lrangle{x,v}+b)\cdot \mathbbm{1}\{x_w \geq \tilde{b}, \|x\| > B_x(d,m,\eps,\delta)\}] \big | \leq 0.1\langle v_{\mathrm{relu}}, v^\perp\rangle    
        \end{align*}
    \end{enumerate}

    To satisfy the first property, it suffices to take $B_x = O\big(\sqrt{d \log(m/\delta)}\big)$ by Gaussian concentration. The second property is satisfied by $B_x = \sqrt{d}\cdot \mathrm{polylog}(1/\eps)$, again by Gaussian concentration, and the facts that $|b| = O\big(\sqrt{\log 1/\eps}\big)$ and $\langle v_{\mathrm{relu}}, v^\perp\rangle \leq \poly(\eps)$. 

    Let $B_y = |B_x| + O\big(\sqrt{\log (1/\eps)} \big)$ be the maximum value of $\sigma(\lrangle{x,v}+b)$ on the event that $\|x\| \leq B_x$. It follows that:
    \[ B_y(d,m,\eps,\delta) = \poly(d, \log m, \log (1/\delta), 1/\eps). \]

    On the other hand, by Lemma \ref{lem:PGD-finite-sample-v-update}, we know that assuming $|y| \leq B_y$, the sample complexity is bounded by some $\poly(d, 1/\delta, 1/\eps, B_y)$. Since $B_y$ has at most poly-log growth in $m$, and $m$ has at most polynomial growth in $B_y$, we can solve for an upper bound on $B_y$ and $m$ which are (at most) polynomial in the values $d, 1/\eps$, and $1/\delta$. 
    
    Finally, it's easy to check that all the lemmas for $\calD$ also hold for the new, truncated distribution, where we condition on the event that $\|x\| \leq B_x$ and $|y| \leq B_y$. Note that most of $v_{\mathrm{relu}}$ is kept, and $v_{\mathrm{noise}}$ becomes smaller as $y'$ is now closer to the target ReLU. 

Finally, we make provable progress on each iteration when $\eta$ is set properly:
\begin{lemma}[same as Lemma \ref{lem:PGD:making-progress}]
    Suppose \cref{assumption:standard} holds where $b$ is sufficiently negative, and suppose $L(w_t, b) > \alpha\cdot\OPT + \eps$ at some iteration $t$. Then, after an iteration of reweighted PGD with $\lambda = 0.9, \rho=0.5$, and $\eta = c_\eta \frac{\|w_t-v\|}{\|v_\mathrm{update}\|}$ for some $c_\eta \leq 0.1$, then:
\end{lemma}
\begin{proof}
    \begin{align*}
        \|w_{t} - v\|^2 - \|w_{t+1} - v\|^2 &= \frac{\lrangle{w_{t+1}, v} - \lrangle{w_t, v}}{2} \\
        & = \frac{1}{2}\bigg[ \lrangle{\frac{w_t + \eta \hat v_\mathrm{update}}{\|w_t + \eta \hat v_\mathrm{update}\|} , v } - \lrangle{w_t,v} \bigg] \\
        & = \frac{\eta \langle \hat v_\mathrm{update}, v \rangle + (1 - \|w_t + \eta \hat v_\mathrm{update}\|) \lrangle{w_t,v}}{2\|w_t + \eta \hat v_\mathrm{update}\| }\\
        & \geq \frac{\eta \langle \hat v_\mathrm{update}, v \rangle - \frac{\eta^2\|\hat v_\mathrm{update}\|^2}{2} \lrangle{w_t,v}}{2\|w_t + \eta \hat v_\mathrm{update}\| }.
    \end{align*}
    When $b$ is sufficiently negative, by \cref{lem:PGD-finite-sample-v-update}, with probability $\geq 1-\delta$ we have $\langle \hat v_\mathrm{update}, v^\perp \rangle \geq 0.9 \|\hat v_\mathrm{update} \|$. Therefore:
    \[ \langle \hat v_\mathrm{update}, v\rangle = \sqrt{1 - \lrangle{w_t,v}^2}\langle \hat v_\mathrm{update}, v^\perp \rangle \geq \|w_t-v\| \cdot 0.9\|\hat v_\mathrm{update}\|\]
    
    Suppose we take $\eta = c_\eta \frac{\|w_t-v\|}{\|\hat v_\mathrm{update}\|}$ for any $ c_\eta \leq 0.1$, then:
    \begin{align*}
        \|w_{t} - v\|^2 - \|w_{t+1} - v\|^2 & \geq \eta \|\hat v_\mathrm{update}\|\cdot \frac{0.9 \|w_t -v\| - \frac{1}{2}\eta \|\hat v_\mathrm{update}\| }{2\|w_t + \eta \hat v_\mathrm{update}\| }\\
        & = c_\eta\|w_t -v\|\cdot \frac{(0.9 - 0.5c_\eta)\|w_t -v\|}{O(1)} \\
        & = \Omega(c_\eta \|w_t -v\|^2). 
    \end{align*}
\end{proof}

Note that $\|\hat v_\mathrm{update} \| = \Theta(\|v_\mathrm{update}\|)$, and the latter is bounded by:
\[ \Omega\big( \sqrt{\alpha\cdot \OPT + \eps} \big)\cdot e^{-O(b^2)} \leq \|v_\mathrm{update}\| \leq O\Big( \|w_t - v\| \Phi\big((1-\rho) b \big)\Big). \]
Moreover, we have $\|w_t - v\| \geq \Omega\big(\sqrt{\alpha \cdot \OPT + \eps} / \sqrt{\Phi(b)}\big)$. Therefore, we have $\frac{\|w_t - v\|}{\|\hat v_{\mathrm{update}}\|} = e^{-\Theta(b^2)} = \poly(\eps)$. Hence by setgin $\eta$ to be some fixed polynomial in $\eps$, we have:
\[ c_\eta = \eta \frac{\|\hat v_{\mathrm{update}}\|}{\|w_t - v\|} = \poly(\eps) \leq -0.1. \]

We hence conclude that $T = \poly(\eps)$ iterations suffice to obtain the desired $\|w_t - v\| \leq O\big(( \alpha\cdot\OPT + \eps) / \Phi(b)\big)$ for some $t\in [T]$, which produces $L(w_t , b) \leq \alpha\cdot \OPT + \eps$ by \cref{lem:lower-bound-w-v}. Plugging this $T$ back to the sample complexity, and our main theorem for PGD (\cref{thm:PGD}) follows. 


%% file: appendix_filtered_PCA.tex
Recall that thresholded PCA outputs the top eigenvector of matrix 
$$ M = \EE \big[ xx^\top \indicator{|y| \geq \tau} \big], $$
where $\tau = \frac{1}{|b|}$ is the threshold. 
We use $\hat M$ to denote the estimation of $M$ from $n$ samples:
\[ \hat M = \frac{1}{n} \sum_{i=1}^n x_i x_i^\top \indicator{|y_i| \geq \tau}. \]

During analysis we will also partition $\RR^d$ into two regions based on the sign of $\lrangle{x,v}+b$, and identify their contribution to $M$ separately:
\[ \begin{cases}
    M_0 = \EE_{(x,y)\sim \calD} \big[xx^\top \indicator{\abs{y} \geq \tau, \lrangle{v,x} + b < 0}\big],\\
    M_1 = \EE_{(x,y)\sim \calD} \big[xx^\top \indicator{\abs{y} \geq \tau, \lrangle{v,x} + b \geq 0}\big].
\end{cases} \]

The rest of this section is dedicated to proving this theorem. In \cref{subsec:PCA-three-lemmas}, we state and prove three lemmas in the following order:
\begin{enumerate}
    \item For any unit vector $u$, $u^\top M_0 u$ is small. (\cref{lem:PCA-lem-1})
    \item For any unit vector $u$ perpendicular to $v$, $u^\top M_1 u$ is small. (\cref{lem:PCA-lem-2})
    \item $v^\top M_1 v$ is large, which helps us identify the true direction. (\cref{lem:PCA-lem-3})
\end{enumerate}
In \cref{subsec:proof-of-PCA-thm}, we give a simple proof of the main theorem using these lemmas. 

\subsection{Technical lemmas.} \label{subsec:PCA-three-lemmas}
The following is used in the proof of Lemma \ref{lem:PCA-lem-1}:
\begin{fact}\label{lem:bounding-x-squared-unlikely-event}
    For all sufficiently small $p$ and any event $E$ with $\PP[E] = p$, we have:
    \[ \EE_{x\sim \calN(0, I_d)}[ \lrangle{x,u}^2 \indicator{E} ] = O\Big( p \log \frac{1}{p} \Big). \]
\end{fact}
\begin{proof}
    Consider event $E_u = \{ |\lrangle{x,u}| \geq |\Phi^{-1}(p / 2)| \}$. Clearly we have $\PP[E_u] = p$, and 
    $$ \EE_{x\sim \calN(0, I_d)}[ \lrangle{x,u}^2 \indicator{E} ] \leq \EE_{x\sim \calN(0, I_d)}[ \lrangle{x,u}^2 \indicator{E_u} ]. $$

    It now suffices to upper bound the right hand side. Consider threshold $t = \Phi^{-1}(p/2)$. Then, because 
    $$\PP[E_u] = 2\Phi(t) = \Theta\Big(\frac{e^{-t^2 / 2}}{t}\Big) = p, $$
    we have $t = O\big( \sqrt{\log (1 / p)}\big)$ for all sufficiently small $p$. Therefore:
    \begin{align*}
        \EE_{x\sim \calN(0, I_d)}[ \lrangle{x,u}^2 \indicator{E_u} ] & = 2\int_t^\infty s^2 d\Phi(s) \\
        & = \Theta\big(t^2 \Phi(t)\big) = O\Big( p \log \frac{1}{p} \Big).
    \end{align*}
\end{proof}

    
    

\begin{lemma}[same as Lemma \ref{lem:PCA-lem-2}]
    For all sufficiently negative $b$, and for any unit vector $u \perp v$, we have:
    \[ u^\top M_1 u \leq \Phi(b). \]
\end{lemma}
\begin{proof}
    Let $x_v,x_u$ be the component of $x$ along $v,u$, respectively. Because $u\perp v$, by the property of isotropic Gaussian, we can integrate along the $u$-direction for each fixed $x_v = \lrangle{x,v}$:
    \begin{align*}
        u^\top M_1 u & = \EE[ \lrangle{x,u}^2 \indicator{|y|\geq \tau, \lrangle{x,v} + b \geq 0} ] \\
        & \leq \EE[ \lrangle{x,u}^2 \indicator{\lrangle{x,v} + b \geq 0} ] \\
        & = \int_{|b|}^\infty \int_{-\infty}^\infty x_u^2  \,d\Phi(x_u) \,d\Phi(x_v) \\
        & = \Phi(b).
    \end{align*}
\end{proof}

The proof of the third lemma follows the intuitive description in \cref{subsec:main-thresholded-PCA}: the adversary can only suppress a fraction of the ReLU below the threshold. 
\begin{lemma}[same as Lemma \ref{lem:PCA-lem-3}]
    Suppose \cref{assumption:standard} holds. For all sufficiently negative $b$, we have:
    \[ v^\top M_1 v = \Omega\big( b^2 \Phi(b)\big). \]
\end{lemma}
\begin{proof}
    For convenience, we will often use random variable $z = \lrangle{x,v}$. Define event $A = \{z +b \geq \tau\}$, the outcomes on which the best-fit ReLU $\sigma(z + b)$ takes value at least $\tau$. It now suffices to show the right hand side of the following expression is at least $\Omega\big(b^2\Phi(b)\big)$:
    $$ v^\top M_1 v = \EE\big[ \lrangle{x,v}^2 \indicator{|y| \geq \tau} \big] \geq \EE\big[ z^2 \indicator{y \geq \tau}\indicator{A} \big]. $$
    
    Moreover, because $z \geq |b|$ on all of $A$, the proof is finished once we show that $\PP[\{y \geq \tau\} \cap A] \geq 0.01 \Phi(b)$. 
    
    Consider an adversary who wants to minimize the probability $\PP[|y|\geq \tau \mid x \in A]$, under the constraint that $\EE\big[\big(y - \sigma(z + b)\big)^2\indicator{A} \big] \leq \OPT$. Note that the adversary's action can be recorded as function $p:[|b|+\tau, \infty) \in [0,1]$, where $p(z') = \PP[y < \tau \mid z = z']$ is the probability that they suppress the $y$ value below threshold, when $\lrangle{x,v} = z'$. 
    
    We first show that the optimal strategy for the adversary is to take $p(z) = \indicator{z \leq t}$, where $t\in [|b| + \tau, \infty)$ is the largest value for which the adversary does not exceed the budget:
    \[ t = \sup_{} \left\{ t'\in [|b|+\tau, \infty): \int_{|b|+\tau}^{|b|+\tau+t} (z+b-\tau)^2\,d\Phi(z) \leq \OPT \right\}, \]
    which indicates the strategy to suppress exactly the \textit{smallest ReLU values}. Consider any $p(z)$ the adversary picks, we WLOG suppose it exhausts all the budget: 
    \[ \int_{|b|+\tau}^\infty (z+b-\tau)^2 p(z)\,d\Phi(z) = \OPT = \int_{|b|+\tau}^t (z+b-\tau)^2\,d\Phi(z), \]
    consequently,
    \begin{equation}\label{eq:PCA-adv-1}
        \int_{|b|+\tau}^t (1 - p(z))(z+b-\tau)^2 p(z)\,d\Phi(z) = \int_t^\infty p(z)(z+b-\tau)^2\,d\Phi(z),
    \end{equation}
    Now, to show $\indicator{z \leq t}$ is optimal, it now suffices to show that it suppresses more contribution from the ReLU than this $p(z)$:
    \begin{equation}\label{eq:PCA-adv-2}
         \int_{|b| + \tau}^t (1 - p(z))z^2\,d\Phi(z) \geq  \int_t^\infty p(z)z^2\,d\Phi(z).
    \end{equation}

    But this is immediately true in light of \cref{eq:PCA-adv-1}: the value $\frac{z^2}{(z+b-\tau)^2}$ is always larger when $z \leq t$ than when $z \geq t$, hence the left hand side is indeed larger in \cref{eq:PCA-adv-2}. Intuitively, the adversary can gain more by paying less when the ReLU value is small. 

    Now, to finish the proof, we have to show that the adversary, even under this optimal strategy, cannot substantially harm the contribution from ReLU. We claim that for large $\alpha$ and sufficiently negative $b$ we have $t \leq \frac{1}{|b|}$, otherwise the adversary would exceed its budget: 
    \begin{align*}
        \int_{|b| + \tau}^{|b| + \tau + \frac{1}{|b|}} (z+b-\tau)^2\,d\Phi(z) &= \Theta \left( e^{-1} \sum_{k=3}^\infty \frac{1}{k!} \right) \int_{|b|+\tau}^\infty (z+b-\tau)^2\,d\Phi(z) \\
        & = \Theta \left( e^{-1} \sum_{k=3}^\infty \frac{1}{k!} \right) \frac{\Phi(b-\tau)}{(b-\tau)^2} \\
        & = \Theta \left( e^{-2} \sum_{k=3}^\infty \frac{1}{k!} \right) \frac{\Phi(b)}{b^2}.
    \end{align*}
    Here the first equality is justified in \cref{appendix:mills-ratio-and-related-lemmas}. Note that $\left( e^{-2} \sum_{k=3}^\infty \frac{1}{k!} \right)$ is a constant independent of $\alpha$, hence by taking $\alpha$ large and $\OPT = O\left( \frac{\Phi(b)}{\alpha b^2}\right)$, we must have $t \leq \frac{1}{|b|}$. The proof is finished by noting that $\Omega(1)$ fraction of the ReLU's contribution is kept when $t \leq \frac{1}{|b|}$:
    \[ \int_{|b|+\tau+t}^\infty z^2\,d\Phi(z) \geq \int_{|b| + \frac{2}{|b|}}^\infty z^2\,d\Phi(z) = \Omega\big(b^2\Phi(b)\big). \]
\end{proof}

\subsection{Proof of main theorem.} \label{subsec:proof-of-PCA-thm}
The main theorem now follows from some linear algebra manipulation and concentration inequality. 


\begin{proof}[Proof of \cref{thm:filteredPCA-main-thm}]
    It's clear now that $M$ has greater magnitude in $v$ than in any perpendicular direction $u\perp v$, for sufficiently large $\alpha$ and negative $b$. We now show that the finite-sample estimation $\hat M$ using $n = \poly(d,1/\eps,1/\delta)$ samples is a good estimator of $M$, in the sense that its top eigenvector still has a dominating component in $v$. 
    
    For each sample $i$, the matrix $\EE[x_i x_i^\top \indicator{|y| \geq \tau}]$ is a $1$-sub-Gaussian matrix, and by an analog of Hoeffding's inequality for matrices (see \cite{Tropp12-matrix-concentration-ineq} for related results):
    \[ \PP \big[ \|\hat{M} - M\|_{op} > t \big] \leq 2d\exp\big( - \Omega(nt^2) \big). \]
    Hence, by taking $n = O\big( \frac{\log(d/\delta)}{\Phi(b)^2} \big) = \poly \big(\log d, \log \frac{1}{\delta}, \frac{1}{\eps}\big)$. we have $\|\hat{M} - M\|_{op} \leq \Phi(b)$ with probability at least $1-\delta$. Combining this with previous lemmas, it follows that with the same probability we have:
    \[ \|\hat M u\| = O\left( \frac{1}{b^2} + \frac{\log \alpha}{\alpha} \right)\|\hat M v\|, \]
    for all unit vector $u\perp v$. 

    Let $w\in \RR^d$ be the output of thresholded PCA. We can decompose it into $w = \lrangle{w,v}\cdot v + \lrangle{w,u}\cdot u$, where $u$ is again a unit vector perpendicular to $v$. Then: 
    \begin{align*}
        \|\hat Mv\| \leq \|\hat Mw\| & \leq \lrangle{w,v}\|\hat Mv\| + \lrangle{w, u}\|\hat Mu\| \\
        & \leq \left( \lrangle{w, v} + \lrangle{w, u} \frac{\|\hat Mu\|}{\|\hat Mv\|} \right) \|\hat Mv\|
    \end{align*}

    Therefore,
    \begin{align*}
        1 \leq \lrangle{w, v} + \lrangle{w, u} \frac{\|\hat Mu\|}{\|\hat Mv\|} \leq \lrangle{w, v} + \frac{\|\hat Mu\|}{\|\hat Mv\|},
    \end{align*}

    which means $\lrangle{w, v} \geq 1 - \frac{\|\hat Mu\|}{\|\hat Mv\|}$. 

    Now the proof is finished by plugging in $b \leq -\sqrt{\alpha / \log \alpha}$, in which case we have:
    \[ \lrangle{w, v} \geq 1 - \frac{\|\hat Mu\|}{\|\hat Mv\|} = 1 - O\left( \frac{\log \alpha}{\alpha} \right), \]
    as desired. 
\end{proof}



\section{Putting things together}\label{sec:appendix-putting-things-together}
Now we state and prove the main algorithmic result of this paper:
\begin{theorem}\label{thm:complete-alg-result}
    There exists a constant $\alpha$, such that for all $W > 0$ the following holds. Let $\calD$ be a the joint distribution of $(x,y)\in \RR^d\times \RR$, where the $x$-marginal is $\calN(0, I_d)$. With population expectations replaced by finite-sample estimates, Algorithm \ref{alg:full-alg} will run in $\poly(d, \frac{1}{\eps}, \frac{1}{\delta}, W)$ time and samples, and with probability at least $1 - \delta$, outputs parameters $\hat w\in \RR^n, \hat b\in \RR$, such that:
    \[ L(\hat w, \hat b) \leq \alpha\cdot\inf_{\substack{w \in \R^d, b \in \R:\\ \norm{w}_2\leq W}} L(w, b) + \eps. \]
\end{theorem}
\begin{proof}    
    The prior works that we use~\citep{AwasthiTV-ICLR23-ReLUGD,DiakonikolasKTZ22-agnostic-learning-unbiased-activations-GD} both guarantee \textit{some} constant approximation factor. We can simply take the maximum between their factors and our algorithm's approximation factor to be the final approximation factor. 

    Regarding the special case of positive bias: the algorithm in \citet{DiakonikolasKTZ22-agnostic-learning-unbiased-activations-GD} works for a class of unbounded activations, which includes positively biased ReLUs. Specifically, for any $b > 0$, the new activation function $\Tilde{\sigma}(\lrangle{x,w}) = \sigma(\lrangle{x,w} + b)$ is a valid activation function under their framework. We note that this requires knowledge about $b$, but we can again do a grid search over the ``guesses'' of $b$ in polynomial time, up to $b\leq \poly(W, \log 1/\eps)$. 
    
    To deal with very positive $b$, it suffices to use the parameters $w,b$ from a certain linear regression variant. In particular, we run linear regression with bias, and we limit bias to be $b \geq \Theta\big(\sqrt{\log(1/\eps)}\big)$. Let $w,b$ be the output of that linear regression problem. Suppose the optimal ReLU is $\sigma(\lrangle{v,x} + b^*)$ with $b^* \geq \Theta\big(\sqrt{\log(1/\eps)}\big)$. Then:
    \begin{align*}
        \EE\big[\big(y - \sigma(\lrangle{w,x}+b) \big)^2 \big] & \leq 2 \EE\big[\big(y - (\lrangle{w,x}+b) \big)^2 \big] + 2\EE\big[\big(\lrangle{w,x}+b - \sigma(\lrangle{w,x}+b) \big)^2 \big]\\
        &= 2 \EE\big[\big(y - (\lrangle{w,x}+b) \big)^2 \big] + \Theta(W^2/b)\Phi(-b/W)\\
        & \leq 2\EE\big[\big(y - (\lrangle{v,x}+b^*) \big)^2 \big] + \Theta(W^2/b)\Phi(-b/W)\\
        & \leq 2\EE\big[\big(y - \sigma(\lrangle{v,x}+b^*) \big)^2 \big] + \Theta(W^2/b)\Phi(-b/W) + \Theta(W^2/b^*)\Phi(-b^*/W),
    \end{align*}
    it follows that the loss of this ReLU candidate is $\leq O(\OPT) + \eps$. 

    Moreover, the algorithm~\citep{AwasthiTV-ICLR23-ReLUGD} for moderately-biased ReLU can solve all constant-bounded bias, for any constant of our choice. The trade-off is that the approximation factor is larger as we allow for larger constants. This guarantee is good enough for our purpose, as our algorithm has provable guarantee when $b \leq b_\alpha < 0$, where $b_\alpha$ is constant. 

    Now we have reduced to the $b\to -\infty$ regime. Since we also consider the zero function $\bfzero$ as a potential output in Algorithm \ref{alg:full-alg}, it also suffices to assume $\bfzero$ incurs loss $> \alpha\cdot \OPT + \eps$. This allows us to apply our main theorems \cref{thm:PGD} and \cref{thm:filteredPCA-main-thm}, which take into account finite-sample estimations. 

    Finally, we bound the time complexity of our grid search approach. Suppose the optimal ReLU is some non-zero function $\sigma(\lrangle{x,v}+b)$. Then, from $\eps = O\big( \Phi(b) / b^2 \big)$ we know $b = O(\sqrt{\log 1 / \eps})$, and we also have $\|v\| \leq W$ by assumption. This means the grid search with accuracy $0.1\sqrt{\eps}$ terminates in $\poly(W, 1/\eps)$ rounds. 
    
    During the grid search, there must be some pair $(\beta_\mathrm{ind}, \gamma_\mathrm{ind})$ of parameters that correctly estimates the optimal $\|v\|, b$ each up to error at most $0.1\sqrt{\eps}$. Since we apply the subroutines in a way that produces a smaller error margin than $\alpha$ and $\eps$:
    \[ L_\mathrm{ind} \leq 0.1\alpha\cdot \OPT_\mathrm{ind} + 0.1\frac{\eps}{\beta_\mathrm{ind}^2}, \]
    this loss will indeed translate into $\alpha\cdot \OPT + \eps$ via \cref{prop:prelim:scaling-v} and \cref{prop:prelims:search-v-b}. 
\end{proof}

%% file: appendix_hardness_results.tex
In this section we show the following theorem:

\begin{theorem}[Same as Theorem~\ref{thm:intro:CSQ}]\label{thm:CSQ}
    There exists a function $F(\eps)$ that goes to infinity as $\eps \to 0$, such that for any $\eps>0$ and any constant $\alpha \ge 1$, there exists a family of instances with $\OPT\le \eps/\alpha$ such that any CSQ algorithm that can agnostically learn an arbitrary ReLU neuron with loss at most $\alpha \cdot \OPT + \eps$ (as defined in \cref{eq:intro:formulation}) must use either $2^{d^{\Omega(1)}}$ queries or queries of tolerance $d^{-F(\eps)}$.
\end{theorem}

In this section we will follow a standard procedure for showing CSQ hardness, which can be found in e.g. \citet{DiakonikolasKPZ-COLT21-OptimalityOfPolyRegression}. We use many lemmas from previous works as black box.  


\subsection{Preliminaries on high-dimensional geometry and CSQ}
Before proving the main theorem, we shall introduce some helpful lemmas from previous works. The first lemma quantifies the fact that, for large $d$, there are exponentially many near-perpendicular directions in $\RR^d$. 
\begin{lemma}[Lemma 3.7 from \citet{DiakonikolasKS17-SQLB-high-dim-Gaussian-Mixture}]\label{lem:CSQ-big-set-small-inner-product}
    For any constant $c\in (0, \frac{1}{2})$, there exists a set $S$ of $2^{\Omega(d^c)}$ unit vectors in $\RR^d$, such that for each distinct $u,v\in S$ we have:
    \[ |\lrangle{u,v}| = O(d^{c-\frac{1}{2}}). \]
\end{lemma}

Next, we introduce some background in CSQ hardness. Most of these definitions and lemmas can be traced back to the seminal work \citet{FeldmanGRVX13-planted-clique-CSQ-hardness-framework}. 

\begin{definition}[CSQ dimension]
Let $\calD_x$ be a distribution over space $\calX$, $\calG$ a set of functions from $\calX$ to $\RR$, and $\beta,\gamma$ two positive parameters. We define the \textit{correlational statistical query dimension of $\calG$ w.r.t. $\calD_x$ with pairwise correlation $(\gamma,\beta)$}, denoted $\mathrm{CSD}_{\calD_x}(\calG, \gamma, \beta)$, to be the largest integer $D$ such that, there is a subset of $D$ functions $\{f_1, \ldots, f_D\}\subseteq \calG$, such that for all $i,j\in[D]$:
\[ \big| \EE_{x\sim \calD_x} [f_i(x)f_j(x)] \big| \leq \begin{cases}
    \gamma & \text{, if } i \neq j,\\
    \beta & \text{, if }i = j.
\end{cases} \]
\end{definition}

\begin{definition}[average CSQ dimension]
    Let $\calD_x$ be a distribution over space $\calX$, $\calG$ a finite set of functions from $\calX$ to $\RR$, and $\gamma$ a positive parameter. We define the average pairwise correlation of functions in $\calG$ to be:
    \[ \rho(\calG) = \frac{1}{|\calG|^2} \sum_{f,g\in \calG} \big| \EE_{x\sim \calD_x} [f(x)g(x)] \big|. \]
    Then, the \textit{average correlational statistical query dimension of $\calG$ w.r.t. $\calD_x$ with parameter $\gamma$}, denoted $\mathrm{CSDA}_{\calD_x}(\calG, \gamma)$, is the largest integer $D$ such that every subset $\calG' \subseteq \calG$ of size at least $\frac{|\calG|}{D}$ has average pairwise correlation $\rho(\calG') \geq \gamma$. 
\end{definition}

The next lemma shows large CSQ dimension implies large average CSQ dimension:
\begin{lemma}\label{lem:CSQ-CSD-implies-CSDA}
    Let $\calD_x$ be a distribution over space $\calX$, and let $\calG$ a finite set of functions from $\calX$ to $\RR$ with $\mathrm{CSD}_{\calD_x}(\calG,  \gamma, \beta) = D$ for some $\gamma, \beta > 0$. Then, for all $\gamma' > 0$, we have:
    \[ \mathrm{CSD}_{\calD_x}(\calG, \gamma + \gamma') \geq \frac{D\gamma'}{\beta - \gamma}. \]
\end{lemma}

Finally, large \textit{Average} CSQ dimension implies the following CSQ hardness result:
\begin{lemma}[Theorem B.1 in \citet{GoelGJKK20-CSQ-hardness-learning-one-layer-NN}]\label{lem:CSQ-hard-to-learn-high-CSDA-family}
    Suppose $x\sim \calD_x$ and let $\calH$ be a real-valued function class that includes the zero-function, with $\EE_{x\sim\calD_x} [f(x)^2] \geq \eta$ for all non-zero $f\in \calH$. Let $D = \mathrm{CSDA}_{\calD_x}(\calH, \gamma)$ for some $\gamma > 0$, then any CSQ algorithm that realizably learns $\calH$ up to $L^2$ error strictly smaller than $\eta$
    needs at least $\frac{D}{2}$ queries or queries of tolerance $\sqrt{\gamma}$. 
\end{lemma}

\subsection{Hermite expansion of negatively biased ReLU}
We first define some new notations. For any $k\in \NN$, let $H_k(x)$ be the $k$th probabilist's Hermite polynomial:
\[ H_k(x) = (-1)^k \exp\Big(\frac{x^2}{2}\Big)\cdot  \frac{d^k}{dx^k} \exp\Big(-\frac{x^2}{2}\Big), \]

Note that $\int_{-\infty}^\infty H_k(s)^2\,d\Phi(s) = k!$, so the set
\[ \left \{\frac{H_k(x)}{\sqrt{k!}}: k\in \NN \right \} \]
forms an orthonormal basis for Hilbert space $L^2_\calN$, the space of all square-integrable functions under to the standard Gaussian measure. We use $\lrangle{f,g}_\calN$ and $\|g\|_\calN$ to denote the inner product and norm of this space, defined by:
\[ \begin{cases}
    \lrangle{f, g}_\calN &:= \int_{-\infty}^\infty f(s)g(s)\,d\Phi(s), \\
    \|f\|_\calN^2 &:= \lrangle{f, f}_\calN.
\end{cases} \]

We are now equipped to analyze the Hermite expansion of negatively-biased ReLUs:
\begin{lemma}[Lemma 3.5 of \citet{AwasthiTV21-learning-depth-2-NN}]\label{lem:CSQ-explicit-Hermite-coefficient}
    Let $f_b(s) = \sigma(s - b)$ for some $b < 0$. Then, 
    $$ \begin{cases}
        \lrangle{f_b, H_0}_\calN &= \varphi(b) - |b|\Phi(b), \\
        \lrangle{f_b, H_1}_\calN &= \Phi(b),
    \end{cases}   $$
    and for all $k \geq 2$:
    \begin{align*}
        \lrangle{f_b, H_k}_\calN &= (-1)^k\cdot H_{k-2}(b)\cdot \varphi(b) \\
        & = H_{k-2}(|b|)\cdot \varphi(b)
    \end{align*}
\end{lemma}

Consequently, as $b\to -\infty$, the $b$-biased ReLU will correlate less with low-degree Hermite polynomials, in the following sense:
\begin{lemma}
    For any fixed $t\in \NN$, Let $f_b(s) = \sigma(s - b)$, then as $b\to -\infty$, we have:
    \[ \sum_{k=0}^t \lrangle{f_b, \frac{H_k}{\sqrt{k!}}}^2 = o(1)\cdot \|f\|_\calN^2. \]
\end{lemma}
\begin{proof}
    By Parseval's identity:
    $$ \sum_{k=0}^\infty \lrangle{f_b, \frac{H_k}{\sqrt{k!}}} = \|f_b\|_\calN^2 = \big( 2+o(1) \big)\frac{\varphi(b)}{|b|^3}. $$
    Applying \cref{lem:CSQ-explicit-Hermite-coefficient}, we have:
    \begin{align*}
        \frac{\sum_{k=0}^t \lrangle{f_b, \frac{H_k}{\sqrt{k!}}}^2 }{\|f_b\|_\calN^2} & = \frac{o(\varphi(b)^2) + \sum_{k=2}^t H_{k-2}(|b|)^2\cdot \varphi(b)^2 / k! }{(2 + o(1)) \varphi(b) / |b|^3} \\
        &= O\big( p(|b|) \big)\cdot \varphi(b),
    \end{align*}
    for some polynomial $p$ of bounded degree. Since $\varphi(b)$ decreases exponentially with $|b|$, this value vanishes as $b \to -\infty$. 
\end{proof}

Finally, we prove the following lemma that is tailor-made for our problem. 
\begin{lemma}[same as Lemma \ref{lem:CSQ-t-eps-unbounded}]
    The following holds for all sufficiently small $\eps$. Let $g_\eps(s) = \sigma(s - b_\eps)$, where $b_\eps$ is chosen so that $\|g_\eps\|_\calN^2 = 3\eps$. 
    Let integer $t_\eps \in \NN$ be:
    \[ t_\eps := \max\left\{t\in\NN: \sum_{k=0}^t \lrangle{g_\eps, \frac{H_k}{\sqrt{k!}}}_\calN^2 \leq  
    \frac{\eps}{\alpha} \right\}, \]
    then, we have $t_\eps \to \infty$ as $\eps \to 0$. 
\end{lemma}
\begin{proof}
    By the previous lemma, for any fixed $t$ we have 
    \[ \sum_{k=0}^t \lrangle{g_\eps, \frac{H_k}{\sqrt{n!}}}_\calN^2 = o(1)\cdot \|g_\eps\|_\calN^2 = o(\eps). \]
    Hence $t_\eps$ cannot be bounded as $\eps \to 0$. 
\end{proof}

\subsection{Proof of main theorem}
To bound the pairwise correlation of two functions defined on almost-perpendicular directions, we introduce one final lemma which we use as a black box:
\begin{lemma}[Lemma 2.3 from \citet{DiakonikolasKPZ-COLT21-OptimalityOfPolyRegression}]\label{lem:CSQ-corrolation-between-gu-gv}
    For function $g:\RR \to \RR$ and unit vectors $u,v\in \RR^d$, we have:
    \[ \underset{x\in \calN(0, I_d)}{\EE} [g(\lrangle{x,u})g(\lrangle{x,v})] \leq \sum_{k=0}^\infty |\lrangle{u,v}|^k \cdot \lrangle{g, \frac{H_k(x)}{\sqrt{k!}}}_\calN^2. \]
\end{lemma}

We are now ready to prove the main hardness result. Suppose some CSQ algorithm $\calA$ can agnostically learn ReLUs with unbounded bias up to error $\alpha\cdot \OPT + \eps$. We first fix some sufficiently small $\eps > 0$, and consider the following set:
\[ \calG = \{G_{v, \eps} = \Tilde{g}_\eps(\lrangle{x,v}): v \in S \}, \]
where $S$ is the set of $2^{d^{\Omega(1)}}$ almost-perpendicular unit vectors in \cref{lem:CSQ-big-set-small-inner-product}, and $\Tilde{g}_\eps$ is a modification of $g_\eps$ in \cref{lem:CSQ-t-eps-unbounded}, by removing the \textit{first $t_\eps$ Hermite components} from the latter:
\[ \tilde g_\eps(s) = \sum_{k=t_\eps+1}^\infty \lrangle{g_\eps, \frac{H_k}{\sqrt{k!}}}_\calN \frac{H_k(s)}{\sqrt{k!}}. \]


We now show that the functions $\{G_{v,\eps}: v\in S\}$ have small pairwise correlation. For each distinct $u,v\in S$:
\begin{align*}
    \underset{x\in \calN(0, I_d)}{\EE}[G_{u,\eps}(x) G_{v,\eps}(x)] & \leq \sum_{k=0}^\infty |\lrangle{u,v}|^k\cdot \lrangle{\tilde g_\eps, \frac{H_k(x)}{\sqrt{k!}}}_\calN^2 \\
    & = \sum_{k=t_\eps+1}^\infty |\lrangle{u,v}|^k\cdot \lrangle{\tilde g_\eps, \frac{H_k(x)}{\sqrt{k!}}}_\calN^2 \\
    & \leq |\lrangle{u,v}|^{t_\eps + 1} \sum_{k=t_\eps+1}^\infty \lrangle{\tilde g_\eps, \frac{H_k(x)}{\sqrt{k!}}}_\calN^2 \\
    & \leq d^{-\Omega(t_\eps)}\cdot \|\tilde g_\eps\|_\calN^2. \\
\end{align*}
Here, the first step is by \cref{lem:CSQ-corrolation-between-gu-gv}, and the last step is by property of $S$ in \cref{lem:CSQ-big-set-small-inner-product}. Since we also have $\underset{x\in \calN(0, I_d)}{\EE} [G_{u,\eps}(x)^2] = \|\tilde g_\eps\|_\calN^2$ by definition, this means:
$$\mathrm{CSD}_{\calN(0,I_d)} \big( \calG, \; d^{-\Omega(t_\eps)}\|\tilde g_\eps\|_\calN^2, \; \|\tilde g_\eps\|_\calN^2 \big) \geq 2^{d^{\Omega(1)}}. $$

Applying \cref{lem:CSQ-CSD-implies-CSDA} with $\gamma = \gamma' = d^{-\Omega(t_\eps)}\|\tilde g_\eps\|_\calN^2$, this translates into:
\begin{align*}
    \mathrm{CSDA}_{\calN(0,I_d)}  \big (\calG, d^{-\Omega(t_\eps)}\|\tilde g_\eps\|_\calN^2 \big) & \geq \frac{2^{d^{\Omega(1)}} \cdot d^{-\Omega(t_\eps)} \cdot \|\tilde g_\eps\|_\calN^2}{(1 - d^{-\Omega(t_\eps)})\|\tilde g_\eps\|_\calN^2} \\
    & = 2^{d^{\Omega(1)}}\text{ for all sufficiently large }d. 
\end{align*}

Now we show that $\calA$ can learn this low-correlation class $\calG$ up to nontrivial accuracy. Given CSQ-oracle access to $x\sim \calN(0, I_d)$ and $y$ labeled by some $G_{v,\eps}(x)$, $\calA$ outputs a function $h$ with error at most:
\begin{align*}
    \underset{x\sim \calN(0, I_d)}{\EE} \big[ \big( h(x) - G_{v,\eps}(x)\big)^2 \big] & \leq \alpha \cdot  \underset{x\sim \calN(0, I_d)}{\EE}  \big[ \big( \sigma(\lrangle{x,v}+b_\eps) - G_{v,\eps}(x)\big)^2 \big] + \eps \\
    & = \alpha\cdot \|g_\eps - \tilde g_\eps\|_\calN^2 + \eps \\
    & \leq \alpha\cdot \frac{\eps}{\alpha} + \eps \\
    & < \left( 3 - \frac{1}{\alpha}\right)\eps\\
    & \leq \underset{x\sim \calN(0, I_d)}{\EE}[G_{v, \eps}(x)^2].
\end{align*}
We can thus apply \cref{lem:CSQ-hard-to-learn-high-CSDA-family} with $\eta = \|\tilde g_\eps\|_\calN = \EE_{x\sim\calN(0, I_d)}[G_{v,\eps}(x)^2]$, and conclude that $\calA$ must use $2^{d^{\Omega(1)}}$ queries or queries of tolerance $d^{-\Omega(t_\eps)}$. Now the proof is finished since $t_\eps \to \infty$ as $\eps \to 0$. 
